\documentclass[preprint,10pt]{elsarticle_mod}
\usepackage{url}
\usepackage{url}
\usepackage{amsmath}
\usepackage{amssymb}
\usepackage{amsthm}
\usepackage{float}
\usepackage{graphics}
\usepackage{graphicx}
\usepackage{framed}
\usepackage{times}
\usepackage{algorithm,algorithmic}
\usepackage{setspace}
\usepackage{tikz}
\usetikzlibrary{arrows}
\usepackage{geometry}
\usepackage{inputenc}
\usepackage{pgfplots}
\usepackage{filecontents}

\newtheorem{theorem}{Theorem}
\newtheorem{Corollary}{Corollary}
\newtheorem{Definition}{Definition}
\newtheorem{Lemma}{Lemma}

\newtheorem{Assumption}[theorem]{Assumption}

\newcommand{\CASE}[1]{\STATE \textbf{case} #1\textbf{} \begin{ALC@g}}
\newcommand{\ENDCASE}{\end{ALC@g}}

\begin{document}

\begin{frontmatter}

\title{On the Robot Assisted Movement in Wireless Mobile Sensor Networks}

\author[mss]{Sajal K. Das\fnref{mssfootnote}}
\ead{sdas@mst.edu}
\author[pwr]{Rafa\l{} Kapelko\corref{cor1}\fnref{pwrfootnote}}
\ead{rafal.kapelko@pwr.edu.pl}
\fntext[mssfootnote]{This work of S. K. Das was partially supported by NSF grants CNS-1850851, OAC-1725755, OAC-2104078, and SCC-1952045.}
\fntext[pwrfootnote]{Supported by Polish National Science Center (NCN) grant 2019/33/B/ST6/02988.}
\cortext[cor1]{Corresponding author at: Department of Computer Science,
Wroc{\l}aw University of Technology, 
 Wybrze\.{z}e Wyspia\'{n}skiego 27, 50-370 Wroc\l{}aw, Poland. Tel.: +48 71 320 33 62; fax: +48 71 320 07 51.}
\address[pwr]{Department of Fundamentals of Computer Science, 
Wroc{\l}aw University of Science and Technology, Poland}
\address[mss]{Missouri University of Science and Technology, USA}
\begin{abstract}
This paper deals with random sensors initially randomly deployed on the line according to general random process
and on the plane according to two independent general random processes.
The mobile robot with carrying capacity $k$  placed at the origin point is to move the sensors to achieve the general scheduling requirement such as coverage, connectivity 
and thus to satisfy the desired communication property in the network.
We study tradeoffs between the  energy consumption in robot's movement, the numbers of sensors $n$, the sensor range $r$, the interference distance $s$, and the robot capacity $k$
until completion of the coverage simultaneously with interference scheduling task.
In this work, we obtain upper bounds for the energy consumption in robot's movement 
and obtain the sharp decrease in the total movement cost of the robot so as to provide the coverage simultaneously with interference
requirement.
\end{abstract} 


\end{frontmatter}
\section{Introduction}
A \textbf{wireless sensor network (WSN)} typically consists of a large number of sensor nodes deployed randomly or according to some well defined distribution over a region of interest
to measure and observe the physical phenomena or events \cite{kumar2007,sajal2008, abbasi2009movement,spa_2013, mohamed2017}. There exists a wide variety of applications of WSNs such us border security and surveillance, wildlife habitat monitoring, 
environmental monitoring (e.g., underwater, earthquake or seismic activities),  intrusion detection, diagnostics in industrial process control, health structure of buildings or bridges, and so on. In some situations, the geographic terrain could be difficult to reach implying the manual deployment of sensors in deterministic patterns
is dangerous or even impossible. Moreover, due to the terrain, wind  and other factors, random deployment may be the only one option.

In WSNs, a fundamental problem to study is the \textbf{energy consumption.} It is well known that the mobile sensors consume much
more energy during the movement than that during the sensing or communication process \cite{ptas}.
Thus, how to minimize the transportation cost of sensors to provide the general scheduling requirement such as connectivity, coverage in the network has great significance.

In this study, we focus on \textit{the reallocation of random sensors} by \textbf{the mobile robot.}
Assume that $n$ sensors are initially randomly deployed on the line  by dropping them from an aircraft according to a general random process
and on the two-dimensional plane according to two independent general random processes.
The robot is placed at a start position and can move the sensors from any initial position to any final position on the line and on the two-dimensional plane.
\textit{The carrying capacity} $k$ of the robot is the maximum number of sensors the robot can carry at a time to deposit all or part of them at any time and to any suitable position it chooses
(see Assumption \ref{assumme:cappacity}).

The aim of this article is to introduce and analyze  the transportation cost for the robot assisted sensor displacement 
and compare it with  the transportation cost for the autonomous sensor displacement 
to provide the general scheduling requirement in the network such as coverage, connectivity, or coverage simultaneously with interference. We compare \textit{the energy consumption in robot assisted sensor displacement} and the case \textit{each sensor moves autonomously} to \textbf{reduce the transportation cost} of mobile sensors
for coverage and interference requirement.
\subsection{Contributions of This Paper}
In this paper we introduce \textbf{novel robot assisted movement} in wireless mobile sensor networks to \textbf{reduce the energy consumption} of transportation cost for sensors. In particular, the mobile robot
with carrying capacity $k$ moves the sensors to provide the general scheduling requirement in the network (see Assumption \ref{assumme:cappacity} in Section \ref{sec:line}).

The objective is to compare the robot assisted transportation cost with the autonomous sensor displacement when the sensors are randomly placed on the line 
(see Definition \ref{def:robot} and Definition \ref{def:movement} in Section \ref{sec:line}) and on the plane
(see Definition \ref{def:t2autonomus} and Definition \ref{def-42} in Section \ref{sec:plane}) according to general random process.
To this aim, we make the following three novel theoretical contribution.
\begin{enumerate} 
\item For both sensors on the line and on the plane
we present the relationship between the transportation cost of the robot and the autonomous sensor displacement which is valid for every 
desired communication property (see Theorem \ref{theorem:robot} in Section \ref{sec:line} and Theorem \ref{theorem:robot_plane} in Section \ref{sec:plane})
\item When the $n$ sensors with identical sensing radius $r_1$ initially randomly displaced on the
$[0,\infty)$ according to Poisson process with arrival rate $\lambda > 0$ (see Definition \ref{def:sensing} and Definition \ref{def:interline} in Subsection \ref{sub:line_a})
we derive tradeoffs between  the energy consumption in robot's movement, the number of sensors $n$, the sensing radius $r_1$, interference distance $s$, the robot capacity $k$ until completion of the
coverage and interference requirement (see Theorem \ref{thm:inspiracja01} in Subsection \ref{sub:line_a} and Table \ref{tab:dwas} in Section \ref{sec:application} for a summary).
\item For $n$ sensors with identical square sensing radius $r_2$ 
that are randomly placed on the plane
according to two identical and independent Poisson processes 
each with arrival rate $\lambda>0$ 
(see Definition \ref{def:square} and Definition \ref{def:squaresquare} in Subsection \ref{sub:plane_a}) 
we derive tradeoffs between the energy consumption in robot's movement, the number of sensors $n$, the square sensing radius $r_2$, interference distance $s$, the robot capacity $k$ until completion of the
coverage and interference requirement (see Theorem \ref{thm:inspiracja02} in Subsection \ref{sub:plane_a} and Table \ref{tab:dwas} in Section \ref{sec:application} for a summary).
\end{enumerate}
\textbf{Table \ref{tab:dwas} summarized the results obtained for the random sensors on the line.}

When the $n$ sensors with identical sensing radius $r_1$ are randomly displaced on the
$[0,\infty)$ according to Poisson process with arrival rate $\lambda=n,$ we have the following results.
\begin{itemize}  
 \item Let $a\in[1,2].$ For the sensing radius $r_1=\frac{1}{2n}$ and the interference distance $s=\frac{1}{n}$, the expected $T^{(a,1)}_{(AS)}$ transportation cost for the autonomous sensor displacement 
 xis $\Theta\left(n^{1-\frac{a}{2}}\right)$
 (see Table \ref{tab:dwasak}); and
 the expected $T^{(a,1)}_{(RS)}$-transportation cost for the robot assisted sensor displacement is $\frac{O\left(n^{\frac{a}{2}}\right)}{k^{a}},$ 
 provided that $n^{1-\frac{1}{a}}\le k\le \sqrt{n}$
 (see Table \ref{tab:dwas}).
 
 Hence the robot with $\lceil\sqrt{n}\rceil$\textbf-\textbf{capacity 
 reduces the transportation cost  from} $\Theta\left(n^{1-\frac{a}{2}}\right)$ \textbf{to constant} $O(1)$ for $(r_1,s)$-coverage and interference requirement
(see Definition \ref{def:interline}).
 \item Let $a\ge 1.$  If the sensing radius $r_1>\frac{1}{2n}$ and the interference distance $s>\frac{1}{n}$, the expected $T^{(a,1)}_{(AS)}$ transportation cost for the autonomous sensor displacement is $\Theta(n)$
 (see Table \ref{tab:dwasak}).
 The expected $T^{(a,1)}_{(RS)}$-transportation cost for the robot assisted sensor displacement is $\frac{O\left(n^a\right)}{k^a},$ 
  provided that $n^{1-\frac{1}{a}}\le k\le n$
 (see Table \ref{tab:dwas}). 
 
 Therefore, 
 the robot with $n$-\textbf{capacity 
 reduces the transportation cost  from} $\Theta(n)$ \textbf{to constant} $O(1)$ for $(r_1,s)$-coverage and interference requirement
(see Definition \ref{def:interline}).
\end{itemize}
\textbf{Table \ref{tab:dwasasekek} summarized the results obtained for the sensors on the plane.}

When $n$ sensors with identical square sensing radius $r_2$ 
are randomly displaced on the plane
according to two identical and independent Poisson processes 
each with arrival rate $\lambda=\sqrt{n},$ we have the following results.
\begin{itemize}  
 \item Let $a\in\left[1,\frac{4}{3}\right].$ For the square sensing radius $r_1=\frac{1}{2\sqrt{n}}$ and the interference distance $s=\frac{1}{n}$, the expected $T^{(a,2)}_{(AS)}$ transportation cost for the autonomous 
 sensor displacement is $\Theta\left(n^{1-\frac{a}{4}}\right)$
 (see Table \ref{tab:dwasak}); and
 the expected $T^{(a,2)}_{(RS)}$-transportation cost for the robot assisted sensor displacement is $\frac{O\left(n^{\frac{3a}{4}}\right)}{k^{a}},$ 
 provided that $n^{1-\frac{1}{a}}\le k\le n^{\frac{1}{4}}$
 (see Table \ref{tab:dwasasekek}).
 
 Hence the robot with $\lceil n^{\frac{1}{4}}\rceil$\textbf-\textbf{capacity 
 reduces the transportation cost  from} $\Theta\left(n^{1-\frac{a}{4}}\right)$ \textbf{to} $O\left(n^{\frac{a}{2}}\right)$  for  $(r_2,s)$-coverage and interference requirement
(see Definition \ref{def:squaresquare}).
 \item Let $a\in\left[1,2\right].$  If the square sensing radius $r_2>\frac{1}{2n}$ and the interference distance $s>\frac{1}{n}$, the expected $T^{(a,1)}_{(AS)}$ transportation cost for the autonomous sensor displacement 
 is $\Theta(n)$
 (see Table \ref{tab:dwasak}).
 The expected $T^{(a,2)}_{(RS)}$-transportation cost for the robot assisted sensor displacement is $\frac{O\left(n^a\right)}{k^a},$ 
 provided that $n^{1-\frac{1}{a}}\le k\le n^{\frac{1}{2}}$
 (see Table \ref{tab:dwasasekek}). 
 
 Therefore, 
 the robot with $\lceil \sqrt{n}\rceil$-\textbf{capacity 
 reduces the transportation cost  from} $\Theta(n)$ \textbf{to constant} $n^{\frac{a}{2}}$ for $(r_2,s)$-coverage and interference requirement
(see Definition \ref{def:squaresquare}).
\end{itemize}

Similar \textbf{decrease in the transportation} cost also holds for all parameters $\lambda>0.$

The rest of the paper is organized as follows. Subsection~\ref{related-work} briefly summarizes the related works.
Section \ref{sec:line} analyzes the robot assisted sensor movement on the line.
Section \ref{sec:plane} deals with the robot assisted sensor reallocation on the plane. 
In Section \ref{sec:application} we apply the upper bounds from previous sections
for robot assisted movement to provide the coverage interference requirement on the line and on the plane.
Section \ref{sec:discussion} presents further insights under various robots assisted scenarios while the numerical evaluation of Algorithms 
\ref{robot_assisted} and \ref{robot_assisted_plane}
are presented in Section \ref{sec:experiments}. The final section offers conclusions.

\subsection{Related Works}
\label{related-work}
Autonomous mobile robots have been extensively studied in the literature (e.g., see \cite{slam, khan_ama, 8301580, oscar, GM-VPC, complete_path}).
The problem of coverage path planning (CPP) for multiple cooperating mobile robots is addressed in~\cite{GM-VPC}.
In \cite{8301580}, the authors advance their previous theoretical work by conducting experiments 
on a generalized coverage optimization algorithm using a team of heterogeneous mobile robots.
The active simultaneous localization and mapping (SLAM) framework for a mobile robot to obtain a collision-free trajectory 
with good performance in SLAM uncertainty reduction and in an area coverage task is presented in \cite{slam}.
In \cite{oscar}, the author studies the problem of patrolling the border with a set of $k$ robots. 
In \cite{complete_path}, they presented a generalized complete coverage path planning (CCPP) algorithm and its implementation for a mobile robot.

The coverage problem in sensor networks has been the subject of extensive interest
(e.g., see \cite{kumar2007, younis2008strategies, sajal2008, tcs2009, percolation, salajkcover, siamcontrol2015, faultdas, li2016, TIAN, faultdas, Dobrev2017, kim2017, trans2018, ZHOU2019}).
The theoretical foundations for $k$-barrier coverage were developed in \cite{kumar2007}, while 
\cite{sajal2008} presents and compares several state-of-the-art algorithms and techniques to address the integrated coverage-connectivity issues in  WSNs. 
The analysis of unreliable sensors in a one-dimensional environment is considered in \cite{siamcontrol2015}. 
In \cite{Dobrev2017}, the authors addressed three optimization problems to achieve weak barrier coverage in WSNs.
The family of problems whose goal is to design a network with maximal connectedness subject to a fixed budget constraint is investigated in \cite{ZHOU2019}.

It is worth mentioning that, our work in the current paper is related to the series of papers of the autonomous sensors displacement (see \cite{spa_2013,kranakis_shaikhet,adhocnow2015_KK,kapelkokranakisIPL, KK_2016_cube,
ICDCNkapelko,pervasiveKAPELKO,kapelko_dmaa, ICDCN2020kapelko, fuchs2020, transaction2021}).
It includes research on barrier and area coverage \cite{spa_2013, kapelkokranakisIPL, KK_2016_cube, kapelko_dmaa, fuchs2020},
on interference \cite{kranakis_shaikhet}, coverage simultaneously with interference \cite{ICDCNkapelko,pervasiveKAPELKO}, interference simultaneously with connectivity \cite{ICDCN2020kapelko}
and coverage-connectivity for the range assignment \cite{transaction2021}.
In \cite{fuchs2020}, the authors revisited the asymptotics of a binomial and a Poisson sum that arose as (average) displacement costs when sensors are randomly placed in anchor positions.

Our investigation of robot assisted sensor displacement is inspired \cite{robot2020}, where the problem of robot
assisted restoration of barrier coverage is introduced in a deterministic
setting. In \cite{robot2020} the authors provide an optimal linear-time offline algorithm that gives a minimum-length trajectory for a single robot
that starts at the end of a barrier and achieves barrier coverage.

The novelty of the current paper lies in the introduction and investigation of the robot assisted model and comparison with the recent results for the autonomous sensor displacement
(see \cite{KK_2016_cube,pervasiveKAPELKO,ICDCN2020kapelko,fuchs2020}).
\section{Analysis of Robot Assisted Movement on the Line}
\label{sec:line}
Let $X_i$ be the position of the $i$-th sensor on the line $[0,\infty).$ Let $n$ sensors $X_1\le X_2\le\dots \le X_n$
be initially randomly placed on the line $[0,\infty)$ according to 
\textit{general random process.}

Initial random placement of the sensors this way
does not guarantee such communication properties
as the coverage requirement, interference requirement, the coverage simultaneously with connectivity, etc.
Therefore, we would like to move the sensors from their initial random location to a new position so as 
to achieve the general scheduling requirement, thus providing the desired communication.

It is assumed that the \textit{mobile robot} located at the origin  moves the sensors.  We define the robot capacity as follows.
\begin{Assumption}[$k$-capacity]
\label{assumme:cappacity}
Assume that $k\in\mathrm{N}$ and $1\le k\le n.$  The carrying capacity $k$ of the robot is the maximum number of sensors the robot can carry at a time to deposit all or part of them at any time and to any suitable position it chooses. 
\end{Assumption}
Obviously the robot can deposit a subset of the sensors it was carrying; it may choose to pick up and carry new ones as long as it \textit{does not exceed} the value $k$.
We now define the transportation for the robot assisted displacement.
\begin{Definition}[$T^{(a,1)}_{(RS)}$-transportation cost]
\label{def:robot}
Let $a\ge 1$ be a constant.
Let $T^{(a,1)}_{(RS)}$ be the total distance \textbf{to the power $\mathbf{a}$} travelled by the robot of $k$-capacity from the origin position
to the final location so as to move the sensors $X_i$  to the arbitrary final position $X_i+M_i,$
provided that $i=1,2,\dots n.$ 
\end{Definition}
In the next two definition, we recall the $T^{(a,1)}_{(AS)}$-transportation cost for the autonomous sensor displacement.
\begin{Definition}[$T^{(a,1)}_{(AS)}$-transportation cost]\label{def:movement}
Let $a\ge 1$ be a constant.
Assume that, for $i=1,2,\dots, n$ the sensor $X_i$ moves autonomously to the position $X_i+M_i.$ 
The $T^{(a,1)}_{(AS)}$-transportation cost for the autonomous sensor displacement is defined as the sum
$T^{(a,1)}_{(AS)}=\sum_{i=1}^n |M_i|^a.$
\end{Definition}
In this section we would like to
compare $T^{(a,1)}_{(RS)}$-transportation cost for the robot assisted displacement with $T^{(a,1)}_{(AS)}$-trans\-por\-tation cost
for the autonomous sensor displacement 
to provide the general scheduling requirement in the network such as coverage,
connectivity, or coverage simultaneously with interference, etc.
\subsection{Greedy Procedure}
\label{sub:preset}
This subsection presents the Greedy Procedure 
that explains the basic relationship between $T^{(1,1)}_{(RS)}$-transportation cost and $T^{(1,1)}_{(AS)}$-trans\-por\-ta\-tion cost
in Algorithm \ref{alg_anchor}.
Algorithm \ref{alg_anchor} concerns the mobile robot located at the position $Y_0$ provided that $Y_0\le Y_1\le Y_2\le\dots \le Y_l.$
The mobile robot collects all sensors $Y_1\le Y_2\le\dots \le Y_l$ 
and moves to the final destination $Y_1+M_1\le Y_2+M_2\le\dots \le Y_l+M_l.$ 
Algoritm \ref{alg_anchor} is \textit{greedy} in the sense that robot located at the position $Y_0$ \textit{collects all sensors} $Y_1\le Y_2\le\dots \le Y_l.$
Its analysis is crucial in deriving the results in the next subsection which also analyzes the mentioned relationship. 

\vskip 0.3cm
\begin{algorithm}[H]
\caption{$GP(Y_0,Y_1, Y_2,\dots Y_l)$, $l\le$ $k$-capacity of robot, }
\label{alg_anchor}
\begin{algorithmic}[1]
 \REQUIRE The initial random location $Y_{0}\le Y_{1}\le Y_{2}\le \dots \le Y_{l}$ of $l+1$ sensors on the $[0,\infty)$;
 the robot with carrying capacity $k$ located at the position $Y_{0}.$
 \ENSURE  The final location of the sensors at the points
 $Y_{0},$ $Y_{1}+M_{1}\le$ $Y_{2}+M_{2}\le \dots$ $\le Y_{l}+M_{l};$ the robot located at the position $Y_{l}.$
 \STATE{the robot starts from $Y_{0}$, moves forward and collect the sensors $Y_{1},Y_{2},\dots, Y_{l};$}
 \STATE{Set $q\leftarrow Y_{l};$ }
 \STATE{The robot displaces the sensors $Y_{1},Y_{2},\dots, Y_{l}$ 
 at the destination $Y_{1}+M_{1}\le$ $Y_{2}+M_{2}\le \dots$ $\le Y_{l}+M_{l};$ }
 \STATE{The robot walks  to the position $q;$}
 \end{algorithmic}
\end{algorithm}
\vskip 0.3cm
The following lemma proves the upper bound on the distance travelled by the robot in Algorithm \ref{alg_anchor}.
\begin{Lemma}
\label{thm:lema}
Fix $l\le k\le n.$
For $i=1,2,\dots, l$ assume that the sensor $Y_i$ moves to the arbitrary position $Y_i+M_i$
and the robot with carrying capacity $k$ is located at the position $Y_{0}.$
Then the  total distance moved by the robot according to Algorithm \ref{alg_anchor} is at most
$$2|M_{1}|+ 2|M_{l}| +3(Y_{l}-Y_{0}).$$
\end{Lemma}
Before proving Lemma \ref{thm:lema}, notice that the movement of the robot with carrying capacity $k\ge l$ in Algorithm \ref{alg_anchor} is upper bounded by the
movements $|M_{1}|,$ $|M_{l}|$ and $(Y_{l}-Y_{0}).$
Also, observe that $|M_{1}|$ and $|M_{l}|$
are the movements of the first and last sensors in the sequence 
$Y_{1}\le Y_{2}\le \dots \le Y_{l}.$
The component  $(Y_{l}-Y_{0})$ is independent on the movement and depends only on the initial random location of the first and last sensors in the sequence
$Y_{0}\le Y_{2}\le \dots \le Y_{l}.$
We are now ready to prove Lemma \ref{thm:lema}.
\begin{proof}
Let $T_{(l)}$ be the  total distance moved by the robot according to Algorithm \ref{alg_anchor}
in our Greedy Procedure. We would like to upper bound $T_{(l)}.$ 
Observe that the worst case occurs when $M_{1}<0$ and $M_{l}>0$,
i.e., $Y_{1}+M_{1}\le Y_{1}$ and $Y_{l}\le Y_{l}+M_{l}.$

In this case the robot first moves left-to-right from $Y_{0}$ to $Y_{l}$ to collect the sensors $Y_1\le Y_2\le\dots\le Y_l.$
Then, the robot moves left-to-right from $Y_l$ to $Y_l+M_l$ and displaces the sensor $Y_l$ at the final position $Y_l+M_l.$
Then, the robot moves right-to-left from $Y_{l}+M_{l}$ to $Y_{1}+M_{1}    $ and displaces the sensors $Y_{l-1}\ge Y_{l-2}\ge\dots\ge Y_1$
at the final positions $Y_{l-1}+M_{l-1}\ge Y_{l-2}+M_{l-2}\ge\dots\ge Y_1+M_1.$
Finally,  the robot moves left-to-right  from $Y_{1}+M_{1}$  to $Y_{l}.$ Hence,
$$
T_{(l)}=(Y_{l}-Y_{0})+2|M_{l}|+(Y_{l}-Y_{1})+ 2|M_{1}|+(Y_{l}-Y_{1})
\le 2|M_{1}|+2|M_{l}|+3(Y_{l}-Y_{0}).
$$
This completes the proof of Lemma \ref{thm:lema}.
\end{proof}
\subsection{Main Results on the Line}
\label{sub:line}
In this subsection we analyze and compare the transportation cost for the robot assisted movement with the autonomous sensor displacement when the sensors are on the line.

Let $X_1\le X_2\le\dots \le X_n$ be the initial random positions of $n$ sensors on the $[0,\infty)$ according to general random process. 
Assume that, for $i=1,2,\dots n$ the sensors $X_i$ moves to the arbitrary final position
$X_i+M_i.$

In this subsection we restrict our analysis to \textit{optimal movement}. Recall that by the simple monotonicity lemma, 
no sensor $X_i$ is ever placed before sensor $X_j,$ for all $i<j,$ i.e.,
$$X_1+M_1\le X_2+M_2\le\dots \le X_n+M_n.$$

Let $k$ be the capacity of the robot located at the origin $0$. Consider the following sequence of greedy algorithms $GM_j(k,n),$ where $j\in\{0,1,\dots,k-1\}$
(see Algorithm \ref{robot_assisted}). 
\begin{algorithm}
\caption{$GM_j(k,n)$ Greedy Movement,\,\, $j\in\{0,1,\dots, k-1\},\,\,$ $k$-capacity of robot, $n$-number of sensors.} 
\label{robot_assisted}
\begin{algorithmic}[1]
 \REQUIRE The initial random location $X_1\le X_2\le \dots \le X_n$ of the $n$ sensors on the $[0,\infty)$ according to general random process.
 The mobile robot with carrying $k$-capacity  located at the origin $0.$
 \ENSURE  The final positions of the sensors at the location\\ $X_1+M_1\le X_2+M_2\le\dots \le X_n+M_n.$
 \IF{$j=0$}
 \STATE{do nothing;}
 \ELSE
 \STATE{Set $X_{0}\leftarrow 0;$ } 
  \STATE{Execute $GP(X_{0}, X_1, X_2,\dots, X_j);$}
 \ENDIF
 \FOR{$i=1$  \TO $\left\lfloor \frac{n-j}{k}\right\rfloor$ } 
 \STATE{Execute $GP\left(X_{j+k(i-1)}, X_{j+k(i-1)+1},\dots, X_{j+k(i-1)+k}\right);$}
 \ENDFOR
 \IF{$j+k\left\lfloor \frac{n-j}{k}\right\rfloor < n$}
 \STATE{Execute $GP\left(X_{j+k\left\lfloor \frac{n-j}{k}\right\rfloor}, 
 X_{j+k\left\lfloor \frac{n-j}{k}\right\rfloor+1},\dots, X_n\right);$}
 \ELSE
 \STATE{do nothing;}
 \ENDIF
\end{algorithmic}
\end{algorithm}

Observe that Algorithm $GM_j(k,n)$ runs in rounds. It is a sequence of Greedy Procedures for $l=j\le k$ when $j>0$ (Steps $(1-6)$; $l=k$ (Steps $(7-9)$); and for $l=n-j-k\left\lfloor \frac{n-j}{k}\right\rfloor\le n-j-k\left( \frac{n-j}{k}-1\right)=k$ 
 when $j+k\left\lfloor \frac{n-j}{k}\right\rfloor < n$ (Steps $(10-11)$).
Thus, Algorithm \ref{robot_assisted} moves the sensors $X_1\le X_2\le \dots \le X_n$ to the final location
 $X_1+M_1\le X_2+M_2\le\dots \le X_n+M_n.$ Hence, it is correct.

We note that the presented relationship between the transportation cost of the robot and the autonomous sensor displacement in Theorem \ref{theorem:robot} is valid for every 
desired communication property.
\begin{theorem}
\label{theorem:robot}
Fix $1\le k\le n.$ 
Let $j\in\{0,1,\dots,k-1\}.$
Let $T^{(a,1)}_{(j,RS)}$ be $T^{(a,1)}_{(RS)}$-trans\-por\-tation cost in algorithm $GM_j(k,n).$ 
Then, $$\min_{0\le j\le k-1}T^{(a,1)}_{(j,RS)}$$ is at most
\begin{equation}
\label{aligna1} 
6^aX^a_n+\left(4\left\lfloor\frac{n}{k}\right\rfloor+16\right)^{a-1}\left(2|M_1|^a+2|M_n|^a+\frac{4T^{(a,1)}_{(AS)}}{k}\right).
\end{equation}
\end{theorem}
\begin{proof}
Let us recall the inequality between general means.
Assume that $a\ge 1.$  and $b_1, b_2, \dots b_l$ are positive. Then
\begin{equation}
\label{eq:mean}
\frac{b_1+b_2+\dots +b_l}{l}\le \left(\frac{b_1^a+b_2^a+\dots +b_l^a}{l}\right)^{\frac{1}{a}}
\end{equation}
(see \cite[Theorem 1, Section 2.14.2]{mitrinovic}).

By Definition~\ref{def:t2autonomus},
\begin{equation}
\label{eq:datask}
T^{(a,1)}_{(AS)}=\sum_{i=1}^{n}|M_i|^a
\end{equation}
is the transportation cost for the autonomous sensor displacement in Algorithm \ref{robot_assisted}.
 Let $T^{(a,1)}_{(j,RS)}$ be $T^{(a,1)}_{(RS)}$-trans\-por\-tation cost in Algorithm $GM_j(k,n)$
 for $j=0,1,\dots,k-1.$ 
 From Lemma \ref{thm:lema}, we have the following upper bound
 $$
 T^{(a,1)}_{(j,RS)}\le
  \left(A_j+B_j+2\sum_{i=1}^{\left\lfloor\frac{n-j}{k}\right\rfloor}\left(\left|M_{j+k(i-1)+1}\right|+\left|M_{j+ki}\right|\right)+C_j+D_j+3X_n\right)^a,
 $$
 where
\begin{align*}
A_j=\begin{cases}
0\,\,\,&\text{if}\,\,\ j=0 ,\\
2|M_1|\,\,\,&\text{if}\,\,\ j>0,
\end{cases}
\,\,\,\,\,\,\,\,\,\,\,\,
B_j=\begin{cases}
0\,\,\,&\text{if}\,\,\ j=0,1 ,\\
2|M_j|\,\,\,&\text{if}\,\,\ j>1,
\end{cases}
\end{align*}

 $$C_j=\begin{cases}
0\,\,\,&\text{if}\,\,\ j+k\left\lfloor \frac{n-j}{k}\right\rfloor=n ,\\
2|M_n|\,\,\,&\text{if}\,\,\ j+k\left\lfloor \frac{n-j}{k}\right\rfloor\le n-1,
\end{cases}
$$
 $$D_j=\begin{cases}
0\,\,\,&\text{if}\,\,\ j+k\left\lfloor \frac{n-j}{k}\right\rfloor=n-1, n,\\
2|M_{j+k\left\lfloor \frac{n-j}{k}\right\rfloor+1|}|\,\,\,&\text{if}\,\,\ j+k\left\lfloor \frac{n-j}{k}\right\rfloor < n-1.
\end{cases}
$$
Applying Inequality (\ref{eq:mean}) we get
 $$
 T^{(a,1)}_{(j,RS)}\le 6^aX_n^a
  + 2^{a-1}\left(A_j+B_j+2\sum_{i=1}^{\left\lfloor\frac{n-j}{k}\right\rfloor}\left(\left|M_{j+k(i-1)+1}\right|+\left|M_{j+ki}\right|\right)+C_j+D_j\right)^a.
 $$
Once again, applying Inequality (\ref{eq:mean}) we have
 \begin{align}
\nonumber T^{(a,1)}_{(j,RS)}&\le6^aX^a_n+\left(4\left\lfloor\frac{n-j}{k}\right\rfloor+16\right)^{a-1}\Big(2\sum_{i=1}^{\left\lfloor\frac{n-j}{k}\right\rfloor}\left(\left|M_{j+k(i-1)+1}\right|^a+\left|M_{j+ki}\right|^a\right)\Big)\\
\label{eq:oszacow} & +\left(4\left\lfloor\frac{n-j}{k}\right\rfloor+16\right)^{a-1}\Big(2(A_j/2)^a+2(B_j/2)^a+2(C_j/2)^a+2(D_j/2)^a\Big).
 \end{align}
Let us now make the following important observation
\begin{equation}
 \sum_{j=0}^{k-1}T^{(a,1)}_{(j,RS)}\le k6^aX^a_n+ \left(4\left\lfloor\frac{n}{k}\right\rfloor+16\right)^{a-1}\\
 \label{eq:sum22}\left( 2k|M_1|^a+4\sum_{i=1}^n|M_i|^a+2k|M_n|^a\right).
 \end{equation}
 Figure \ref{fig:algos} illustrates Inequality (\ref{eq:oszacow}) for $n=8$ and $k=3.$
 Namely $T_{(j,RS)}$ is the transportation cost Algorithm $GM_j(k,8)$ for 
 $j=0,1,\dots,3.$
 Let us recall that,
   for $i=1,2,\dots, 8$ the sensor $X_i$ moves to the position $X_i+M_i.$ 
 In Figure \ref{fig:algos}, the black dots indicate the movements $2|M_i|^a$ which consider  in the upper bound 
  estimation (\ref{eq:oszacow}) of  $T^{(a,1)}_{(j,RS)}$ for Algorithm $GM_j(k,8),$ provided $j=0,1,\dots,3.$  
  (Obviously, the white dots indicate the movements $2|M_i|^a$ which don't appear 
  in the upper bound 
  estimation (\ref{eq:oszacow}) of  $T^{(a,1)}_{(j,RS)}$ for Algorithm $GM_j(k,8),$ provided $j=0,1,\dots,3$). 
   Now, two times sum of the black dots in Figure \ref{fig:algos}
   is equal to $$6|M_1|^a+4\sum_{i=1}^8|M_i|^a+6|M_8|^a.$$
   This confirms the movements $|M_1|^a,$ $|M_2|^a,\dots$ $|M_n|^a$ in the upper bound Inequality (\ref{eq:sum22}) for $n=8$ and $k=3$ very well.
    \vspace{-5pt}
  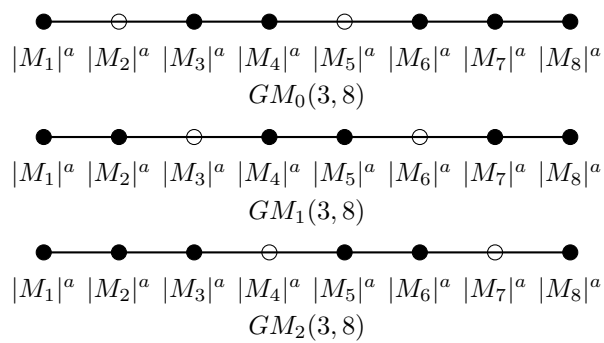
\begin{figure}[H]
  \vskip 0.3cm
  \center
 \begin{tikzpicture}
\filldraw[fill=black] (0,0) circle (0.1cm);
\filldraw[fill=white] (1,0) circle (0.1cm);
\filldraw[fill=black] (2,0) circle (0.1cm);
\filldraw[fill=black] (3,0) circle (0.1cm);
\filldraw[fill=white] (4,0) circle (0.1cm);
\filldraw[fill=black] (5,0) circle (0.1cm);
\filldraw[fill=black] (6.0,0) circle (0.1cm);
\filldraw[fill=black] (7,0) circle (0.1cm);
\draw[thick,] (0,0) -- (7,0);
\draw (0,-0.5) node {$|M_1|^a$};
\draw (1,-0.5) node {$|M_2|^a$};
\draw (2,-0.5) node {$|M_3|^a$};
\draw (3,-0.5) node {$|M_4|^a$};
\draw (4,-0.5) node {$|M_5|^a$};
\draw (5,-0.5) node {$|M_6|^a$};
\draw (6,-0.5) node {$|M_7|^a$};
\draw (7,-0.5) node {$|M_8|^a$};
\draw (3.5,-1) node {$GM_0(3,8)$};
\end{tikzpicture}

\vskip 0.1cm
\begin{tikzpicture}
\filldraw[fill=black] (0,0) circle (0.1cm);
\filldraw[fill=black] (1,0) circle (0.1cm);
\filldraw[fill=white] (2,0) circle (0.1cm);
\filldraw[fill=black] (3,0) circle (0.1cm);
\filldraw[fill=black] (4,0) circle (0.1cm);
\filldraw[fill=white] (5,0) circle (0.1cm);
\filldraw[fill=black] (6,0) circle (0.1cm);
\filldraw[fill=black] (7,0) circle (0.1cm);
\draw[thick,] (0,0) -- (7,0);
\draw (0,-0.5) node {$|M_1|^a$};
\draw (1,-0.5) node {$|M_2|^a$};
\draw (2,-0.5) node {$|M_3|^a$};
\draw (3,-0.5) node {$|M_4|^a$};
\draw (4,-0.5) node {$|M_5|^a$};
\draw (5,-0.5) node {$|M_6|^a$};
\draw (6,-0.5) node {$|M_7|^a$};
\draw (7,-0.5) node {$|M_8|^a$};
\draw (3.5,-1) node {$GM_1(3,8)$};
\end{tikzpicture}
\vskip 0.1cm
\begin{tikzpicture}
\filldraw[fill=black] (0,0) circle (0.1cm);
\filldraw[fill=black] (1,0) circle (0.1cm);
\filldraw[fill=black] (2,0) circle (0.1cm);
\filldraw[fill=white] (3,0) circle (0.1cm);
\filldraw[fill=black] (4,0) circle (0.1cm);
\filldraw[fill=black] (5,0) circle (0.1cm);
\filldraw[fill=white] (6,0) circle (0.1cm);
\filldraw[fill=black] (7,0) circle (0.1cm);
\draw[thick,] (0,0) -- (7,0);
\draw (0,-0.5) node {$|M_1|^a$};
\draw (1,-0.5) node {$|M_2|^a$};
\draw (2,-0.5) node {$|M_3|^a$};
\draw (3,-0.5) node {$|M_4|^a$};
\draw (4,-0.5) node {$|M_5|^a$};
\draw (5,-0.5) node {$|M_6|^a$};
\draw (6,-0.5) node {$|M_7|^a$};
\draw (7,-0.5) node {$|M_8|^a$};
\draw (3.5,-1) node {$GM_2(3,8)$};
\end{tikzpicture}
 \caption{Movement of $n=8$ sensors for the robot with carrying $3-$capacity according to Algorithm $GM_j(3,8),$ when $j\in\{0,1,2\}.$}
  \label{fig:algos}
\end{figure}
Note that the minimal cost satisfies
 $$\min_{0\le j\le k-1}T^{(a,1)}_{(j,RS)}\le\frac{T^{(a,1)}_{(0,RS)}+T^{(a,1)}_{(1,RS)}+\dots+T^{(a,1)}_{(k-1,RS)}}{k}.$$ 
 Hence, applying Inequality (\ref{eq:sum22}) and Formula (\ref{eq:datask}) we have
 $$6^aX^a_n+\left(4\left\lfloor\frac{n}{k}\right\rfloor+16\right)^{a-1}\left(2|M_1|^a+2|M_n|^a+\frac{4T^{(a,1)}_{(AS)}}{k}\right).$$
This is enough to complete the proof of Theorem \ref{theorem:robot}. 
\end{proof}
We can also prove \textbf{the lower bound.}
\begin{Corollary}
 \label{thm:lower}
 Fix $1\le k\le n.$ Assume that $n$ sensors
 $X_1\le X_2\le \dots \le X_n$ are initially randomly placed on the $[0,\infty)$ according to general random process
 and the mobile robot with carrying $k$-capacity  is located at the origin $0.$
 Then, $T^{(a,1)}_{(RS)}$-transportation cost according to every robot assisted algorithm is at least
 $X^a_n.$
\end{Corollary}

\begin{proof}
The robot in the origin point $0$  has to collect all the sensors $X_1\le X_2\le\dots \le X_n.$ Hence, the robot has to move from the origin point $0$ to the random point $X_n.$
Therefore $T^{(a,1)}_{(RS)}$ is at least
$\left(X_n-0\right)^a=X^a_n.$ 
\end{proof}
\section{Analysis of Robot Assisted Movement on the Plane}
\label{sec:plane}
In this section we look at the transportation cost for the robot assisted movement when the sensors are on the plane.

More precisely, we consider $n$ sensors placed in the quadrant $[0,\infty)\times[0,\infty)$
according to two independent general random processes $X_i$ and $Y_i.$ The random position of the sensor on the plane
is determined by the pair $(X_{i_1},Y_{i_2}),$ where $i_{1},i_{2}=1,2,\dots, \sqrt{n}$ and $n=m^2$ for some $m\in\mathbf{N}.$

We are now ready to define the transportation cost for robot assisted displacement and the transportation cost for the autonomous sensor displacement
on the plane.
\begin{Definition}[$T^{(a,2)}_{(RS)}$-transportation cost]
\label{def:t2autonomus}
Let $a\ge 1$ be a constant.
Let $T^{(a,2)}_{(RS)}$ be the total distance \textbf{to the power} $\mathbf{a}$ travelled by the robot of $k$-capacity from the origin position $(0,0)$
to the final location so as to move the sensor $(X_{i_1},Y_{i_2}),$ to the arbitrary final position 
$(X_{i_1}+M_{i_1},Y_{i_2}+N_{i_2}),$ provided that $i_{1},i_{2}=1,2,\dots, \sqrt{n}.$
\end{Definition}
\begin{Definition}[$T^{(a,2)}_{(AS)}$-transportation cost]\label{def-42}
Let $a\ge 1$ be a constant.
Assume that, for $i=1,2,\dots, n$ the sensor 
$(X_{i_1},Y_{i_2})$
moves autonomously to the arbitrary final position $(X_{i_1}+M_{i_1},Y_{i_2}+N_{i_2})$ for $i_{1},i_{2}=1,2,\dots, \sqrt{n}.$
The $T^{(a,2)}_{(AS)}$-transportation cost for the autonomous sensor displacement
is defined as the sum
$T^{(a,2)}_{(AS)}=\sum_{i_1, i_2=1}^{\sqrt{n}} \left(|M_{i_1}|^a+|N_{i_2}|^a\right).$
\end{Definition}
\subsection{Main Results on the Plane}
\label{sub:plane}
In this subsection we analyze and compare the transportation cost for the robot assisted movement with the autonomous sensor displacement when the sensors are on the plane. 
In order to move the sensors on the plane we present the sequence of greedy algorithms 
$GM_{j_1,j_2}(k,n),$ where $j_1,j_2\in\{0,1,\dots, k-1\},\,\,$ (see Algorithm \ref{robot_assisted_plane}).
Algorithm \ref{robot_assisted_plane} works in two phases. During the first phase (Steps $(1-5)$), we execute
Algorithm \ref{alg_anchor} according to the first coordinate. Thus, in the second phase (Steps $(7-11)$), 
Algorithm \ref{alg_anchor} is executed according to the second coordinate. Thus, the robot movements on the plane
is reduced to the robot movements on the line. Hence, Algorithm \ref{robot_assisted_plane} is correct. 
We  apply Theorem \ref{theorem:robot} on the line  from Subsection \ref{sub:line} 
and prove Theorem \ref{theorem:robot_plane} on the plane. 

Figure \ref{rand:pr} illustrates
the trajectory of the robot in Steps $(1-5)$
of Algorithm \ref{robot_assisted_plane} for $n=9$ and $k=3.$
The robot's trajectory is the line segment from $(0,0)$ to $(0,Y_1);$ 
$(0,Y_1)$ to $(\max(X_3,X_3+M_3),Y_1);$ 
$(\max(X_3,X_3+M_3),Y_1)$  to $(0,Y_1);$
$(0,Y_1),$ to $(0,Y_2);$
$(0,Y_2)$ to\\ $(\max(X_3,X_3+M_3),Y_2);$ 
$(\max(X_3,X_3+M_3),Y_2)$  to $(0,Y_2);$
$(0,Y_2),$ to $(0,Y_3);$
$(0,Y_3)$ to $(\max(X_3,X_3+M_3),Y_3)$ and
$(\max(X_3,X_3+M_3),Y_3)$  to $(0,Y_3).$ 
\begin{algorithm}
\caption{$GM_{j_1,j_2}(k,n)$ Greedy Movement,
$j_1,j_2\in\{0,1,\dots, k-1\},\,\,$ $k$-capacity of robot, $1\le k \le \sqrt{n},$ $n$-number of sensors, $n=m^2$ for some $m\in\mathbf{N}$. } 
\label{robot_assisted_plane}
\begin{algorithmic}[1]
 \REQUIRE The initial random location $(X_{i_1},Y_{i_2})$ of the $n$ sensors in the $[0,\infty)\times[0,\infty),$ where $i_1,i_2=1,2,\dots \sqrt{n}$  according to two general random processes.
 The mobile robot with carrying $k$-capacity is located at the origin $(0,0).$
 \ENSURE  The final positions of the sensors at the location\\ $(X_{i_1}+M_{i_1},Y_{i_2}+N_{i_2}),$ where $i_1,i_2=1,2,\dots, \sqrt{n}.$
 \FOR{$i_2=1$  \TO $\sqrt{n}$ } 
 \STATE{The robot walks to the position $(0,Y_{i_2});$}
 \STATE{Execute algorithm $GM_{j_2}(k,\sqrt{n})$ according to the first\\ coordinate of the sensors $(X_1,Y_{i_2}),$ $(X_2,Y_{i_2}),\dots,$\\ $(X_{\sqrt{n}},Y_{i_{2}}).$
 It means to move to the positions $(X_1+M_1,Y_{i_2}),$ $(X_2+M_2,Y_{i_2}),\dots,$ $(X_{\sqrt{n}}+M_{\sqrt{n}},Y_{i_{2}});$}
 \STATE{The robot walks to the position $(0,Y_{i_2});$}
 \ENDFOR
  \STATE{The robot walks to the origin $(0,0);$}
   \FOR{$i_1=1$  \TO $\sqrt{n}$ } 
 \STATE{The robot walks to the position $(X_{i_1}+M_{i_1}, 0);$}
 \STATE{Execute algorithm $GM_{j_1}(k,\sqrt{n})$ according to the second coordinate of the sensors $(X_{i_1}+M_{i_1},Y_1),$ $(X_{i_1}+M_{i_1},Y_2),\dots,$ $(X_{i_1}+M_{i_1},Y_{\sqrt{n}}).$ 
 It means to move to the positions\\  $(X_{i_1}+M_{i_1},Y_1+N_1),$ $(X_{i_1}+M_{i_1},Y_2+N_2),\dots,$\\ $(X_{i_1}+M_{i_1},Y_{\sqrt{n}}+N_{\sqrt{n}});$}
 \STATE{The robot walks to the position $(X_{i_1}+M_{i_1},0);$}
 \ENDFOR
\end{algorithmic}
\end{algorithm}

\begin{figure}[H]
\begin{center}
\begin{tikzpicture}
\draw[thick,dotted] (0,0) -- (4.3,0);
\draw[thick,dotted] (0,3.5) -- (0,3.8);
\draw (0,-0.5) node {$(0,0)$};
\draw (1,-0.5) node {$X_1$};
\draw (2,-0.5) node {$X_2$};
\draw (4,-0.5) node {$X_3$};
\draw (-0.5,0.5) node {$Y_1$};
\draw (-0.5,1.5) node {$Y_2$};
\draw (-0.5, 3.5) node {$Y_3$};

\filldraw[fill=black] (1,0.5) circle (0.06cm);
\filldraw[fill=black] (2,0.5) circle (0.06cm);
\filldraw[fill=black] (4,0.5) circle (0.06cm);
\draw [thick,<-] (3,0.5)--(3.5,0.5);
\draw [thick,->] (0,0.5)--(0.5,0.5);

\filldraw[fill=black] (1,1.5) circle (0.06cm);
\filldraw[fill=black] (2,1.5) circle (0.06cm);
\filldraw[fill=black] (4,1.5) circle (0.06cm);
\draw [thick,<-] (3,1.5)--(3.5,1.5);
\draw [thick,->] (0,1.5)--(0.5,1.5);

\filldraw[fill=black] (1,3.5) circle (0.06cm);
\filldraw[fill=black] (2,3.5) circle (0.06cm);
\filldraw[fill=black] (4,3.5) circle (0.06cm);
\draw [thick,<-] (3,3.5)--(3.5,3.5);
\draw [thick,->] (0,3.5)--(0.5,3.5);

\draw[thick] (0,0) -- (0,3.5);
\draw[thick] (0,0.5) -- (4,0.5);
\draw[thick] (0,1.5) -- (4,1.5);
\draw[thick] (0,3.5) -- (4,3.5);

\draw [thick,->] (0,0)--(0,0.25);
\draw [thick,->] (0,0.75)--(0,1);
\draw [thick,->] (0,2)--(0,2.5);

\end{tikzpicture}
\caption{The trajectory of robot in Steps $(1-5)$ of Algorithm \ref{robot_assisted_plane} for $n=9.$}
\label{rand:pr}
\end{center}
\end{figure}
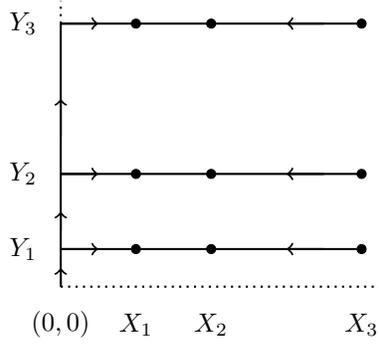
\vspace{-10pt}
\begin{theorem}
\label{theorem:robot_plane}
Fix $1\le k\le \sqrt{n}.$
Let $j_1, j_2\in\{0,1,\dots,k-1\}.$
Let $T^{(a,2)}_{(j_1,j_2,RS)}$ be 
$T^{(a,2)}_{(RS)}$-transportation cost 
in algorithm $GM_{j_1,j_2}(k,n).$ 
Then,
$$\min_{0\le j_1,j_2\le k-1}T^{(a,2)}_{(j_1,j_2,RS)}$$ is at most
\begin{align}
\nonumber&4^an^{a/2-1/2}\left(4\left\lfloor\frac{\sqrt{n}}{k}\right\rfloor+16\right)^{a-1}\frac{T^{(a,2)}_{(AS)}}{k}
+4^{a-1}n^{a/2}\left(4\left\lfloor\frac{\sqrt{n}}{k}\right\rfloor+16\right)^{a-1}\left(2|M_1|^a+2|M_{\sqrt{n}}|^a\right)\\
\nonumber&+4^{a-1}n^{a/2}\left(4\left\lfloor\frac{\sqrt{n}}{k}\right\rfloor+16\right)^{a-1}\left(2|N_1|^a+2|N_{\sqrt{n}}|^a\right)
\nonumber+\left(2^{3a-1}+6^a4^{a-1}n^{a/2}\right)Y^a_{\sqrt{n}}\\&+\left(4^{a-1}+6^a4^{a-1}n^{a/2}\right)X^a_{\sqrt{n}}.
\label{eq:chojnik}
\end{align}
\end{theorem}
\begin{proof}
Let $j_1, j_2\in\{0,1,\dots,k-1\}$ and let $T^{(a,2)}_{(j_1,j_2,RS)}$ be
$T^{(a,2)}_{(RS)}$-trans\-por\-tation cost in algorithm
$GM_{j_1,j_2}(k,n).$ 

Let $D_{1}(j_2,i_2)$ be the distance travelled by the robot 
in Step $(3)$ of Algorithm
\ref{robot_assisted_plane}, provided that $i_2\in\{1,2,\dots,\sqrt{n}\}.$

Let $D_{2}(j_1,i_1)$ be the distance travelled by the robot 
in Step $(9)$ of Algorithm
\ref{robot_assisted_plane}, provided that $i_1\in\{1,2,\dots,\sqrt{n}\}.$

Hence, the distance travelled by the robot in Steps $(1-5)$ of Algorithm \ref{robot_assisted_plane} is equal to 
\begin{equation}
\label{eq:dist01}
\sqrt{n}D_1(j_2,1)+Y_{\sqrt{n}}
\end{equation}
and the distance travelled by the robot is Steps $(7-11)$ of Algorithm \ref{robot_assisted_plane} is equal to 
\begin{equation}
\label{eq:dist02}
\sqrt{n}D_2(j_1,1)+X_{\sqrt{n}}.
\end{equation}
It is easy to see that  the distance travelled by the robot in Step $(6)$ of Algorithm \ref{robot_assisted_plane} is equal to
\begin{equation}
 \label{eq:step6}
 Y_{\sqrt{n}}.
\end{equation}
Combining Expressions (\ref{eq:dist01}-\ref{eq:step6}) we have
$$
T^{(a,2)}_{(j_1,j_2,RS)}=\left(\sqrt{n}D_1(j_2,1)+\sqrt{n}D_2(j_1,1)+2Y_{\sqrt{n}}+X_{\sqrt{n}}\right)^a.
$$
Applying Inequality  (\ref{eq:mean}) we get
\begin{equation}
\label{eq:astra}
T^{(a,2)}_{(j_1,j_2,RS)}=4^{a-1}\left(n^{a/2}D^a_1(j_2,1)+n^{a/2}D^a_2(j_1,1)+2^aY^a_{\sqrt{n}}+X^a_{\sqrt{n}}\right).
\end{equation}
Applying Theorem \ref{theorem:robot} for $n:=\sqrt{n}$ and the sensors $X_1, X_2,\dots, X_{\sqrt{n}}$ we have
\begin{equation}
\min_{0\le j_2\le k-1}D^a_1(j_2,1)\le 6^aX^a_{\sqrt{n}}+
\label{eq:100ac}\left(4\left\lfloor\frac{\sqrt{n}}{k}\right\rfloor+16\right)^{a-1}\left(2|M_1|^a+2|M_{\sqrt{n}}|^a+\frac{4\sum_{i=1}^{\sqrt{n}}|M_i|^a}{k}\right).
\end{equation}
Using Theorem \ref{theorem:robot} for $n:=\sqrt{n}$ and the sensors $Y_1, Y_2,\dots, Y_{\sqrt{n}}$ we get
\begin{equation}
\min_{0\le j_1\le k-1}D^a_1(1,j_1)\le 6^aY^a_{\sqrt{n}}+
\label{eq:100ad}\left(4\left\lfloor\frac{\sqrt{n}}{k}\right\rfloor+16\right)^{a-1}\left(2|N_1|^a+2|N_{\sqrt{n}}|^a+\frac{4\sum_{i=1}^{\sqrt{n}}|N_i|^a}{k}\right).
\end{equation}
Finally, combining Estimations $(\ref{eq:astra}-\ref{eq:100ad})$ we have
the desired upper bound (\ref{eq:chojnik}). 
This completes the proof of Theorem \ref{theorem:robot_plane}.
\end{proof}
We can also prove \textbf{the lower bound.}
\begin{Corollary}
 \label{thm:lowerplane}
 Assume that $n$ sensors are placed on the plane $[0,\infty)\times[0,\infty)$
according to two identical and independent Poisson processes $X_i$ and $Y_i,$
for $i=1,2,\dots, \sqrt{n}$ each with arrival rate $\lambda>0$ 
 and the mobile robot with carrying $k$-capacity  is located at the origin $(0,0).$ 
 The random position of the sensor on the plane
is determined by the pair $(X_{i_1},Y_{i_2}),$ where $i_{1},i_{2}=1,2,\dots, \sqrt{n}$ and $n=m^2$ for some $m\in\mathbf{N}.$ 
 Then, $T^{(a,1)}_{(RS)}$-transportation cost according to every robot assisted algorithm is at least
 $X^a_n.$
\end{Corollary}

\begin{proof}
The robot in the origin point $(0,0)$  has to collect all the sensors $(X_{i_1},Y_{i_2}),$ where $i_{1},i_{2}=1,2,\dots, \sqrt{n}$ and $n=m^2$ for some $m\in\mathbf{N}.$ 
Hence, the robot has to move from the origin point $(0,0)$ to the random point $(X_{\sqrt{n}},Y_{\sqrt{n}}).$
Therefore $T^{(a,1)}_{(RS)}$ is at least
$\left(\sqrt{n}\left(X_{\sqrt{n}}-X_1\right)\right)^a=n^{a/2}\left(X_{\sqrt{n}}-X_1\right)^a.$ 
\end{proof}

\section{Application to Coverage and Interference in Sensor Networks}
\label{sec:application}
In this section we apply the upper bounds  obtained from Theorem \ref{theorem:robot} and Theorem \ref{theorem:robot_plane} 
to the coverage simultaneously with the interference problem. Namely, we compare the known results for the energy consumption in autonomous sensor displacement to fulfil
the coverage and interference requirement obtained in \cite{ICDCNkapelko}, \cite{pervasiveKAPELKO}, \cite{ICDCN2020kapelko} 
with the current estimations in Sections \ref{sec:line} and \ref{sec:plane}.
As the result we obtain novel upper bounds for the energy consumption in robot assisted displacement to provide the coverage and interference requirement. 
\subsection{Coverage and Interference on the Line}
\label{sub:line_a}
In this subsection we look at the $(r_1,s)$-coverage and interference problem on the line. 
Let us recall the definition of the sensing radius.
\begin{Definition}[Sensing Radius]
\label{def:sensing}
We assume that a sensor placed at location $x$ on the line can sense any point at distance $r_1$ at most reither to the left or right of $x$
and call $r_1$ the sensing radius of the sensor.
\end{Definition}
In particular, we consider $n$ random sensors $X_1\le X_2 \le \dots \le X_n$  with identical sensing radius $r_1$
initially randomly displaced on the $[0,\infty)$ according to Poisson process with arrival rate $\lambda>0.$ 
The formal definition of $(r_1,s)$-coverage and interference problem on the line is as follows.
\begin{Definition}[$(r_1,s)$-coverage and interference]
\label{def:interline}
We move the sensors from the initial random locations $X_1\le X_2 \le \dots \le X_n$ to the final positions $X_1+M_1\le$ $X_2+M_2\le \dots \le X_n+M_n$ so as to: 
\begin{itemize}
\item Ensure coverage in the sense that every point in the interval $[0,X_n+M_n]$ is in the sensing radius of at least one sensor.  
\item Avoid interference, i.e., each pair of sensors is at interference distance greater or equal to $s.$
\end{itemize}
\end{Definition}
In the recent paper \cite{ICDCNkapelko}, the maximum of the expected sensor's displacements to the power $a\ge 1$ metric given by
\begin{equation}
\label{eq:icdcna01}
\mathbf{E_{\max}}\left[T^{(a,1)}_{(AS)}\right]=\max_{1\le i\le n}\mathbf{E}\left[|M_i|^a\right]
\end{equation}
was investigated for $(r_1,s)$-coverage and interference problem when the sensors autonomously move on the line.

Further, the author in \cite{ICDCN2020kapelko} considered the expected $a-$total movement as follows
\begin{equation}
\label{eq:icdcna}
\mathbf{E}\left[T^{(a,1)}_{(AS)}\right]=\sum_{i=1}^{n}\mathbf{E}\left[|M_i|^a\right],\,\,\,\text{provided that}\,\,\, a\ge 1.
\end{equation}
for interference-connectivity problem for the autonomous sensor displacement on the line and in the higher dimension.

It is worth pointing out that the main results of \cite{ICDCN2020kapelko} for interference-connectivity problem are easily applicable to
$(r_1,s)$-coverage and interference problem when the sensors move autonomously on the line. 

Table \ref{tab:dwasak} for $d=1$ summarizes the known results on the line.
\begin{table}[H]
\caption {The minimal maximum of expected sensor's displacements and the expected minimal
$T^{(a,d)}_{(AS)}$-transportation cost for  $(r_d,s)$-coverage and interference problem 
of $n$ sensors in the $[0,\infty)^d$ for $d\in\{1,2\}$ provided that $\epsilon\ge \delta>0$ are arbitrarily small constants independent on $\lambda$ and $n$ (see \cite{ICDCNkapelko}, \cite{pervasiveKAPELKO}, \cite{ICDCN2020kapelko}).
} 
\label{tab:dwasak}
\begin{center}
 \begin{tabular}{|c|c|c|c|c|} 
 \hline
 \begin{tabular}{c}Sensing radius $r_1,$\\ Square sensing radius $r_2$ \end{tabular} & Interference distance $s$  & 
 $\mathbf{E_{\max}}\left[T^{(a,d)}_{(AS)}\right]$ & $\mathbf{E}\left[T^{(a,d)}_{(AS)}\right]$ \\
  \hline
 $r_d=\frac{1+\epsilon}{2\lambda}$  & $s=\frac{1-\delta}{\lambda}$  & $\Theta(1)/\lambda^a$ & $O(n)/\lambda^a$ \\ 
 \hline
 $r_d=\frac{1}{2\lambda},$ & $s=\frac{1}{\lambda}$  & $\Theta\left(n^{\frac{a}{2d}}\right)/\lambda^a$ & $\Theta\left(n^{1+\frac{a}{2d}}\right)/\lambda^a$  \\  
 \hline
 $r=\frac{1+\epsilon}{2\lambda}$ & $s=\frac{1+\delta}{\lambda}$ & $\Theta\left(n^{\frac{a}{d}}\right)/\lambda^a$ & 
$\Theta\left(n^{1+\frac{a}{d}}\right)/\lambda^a$\\ 
 \hline
 \end{tabular}
\end{center}
\end{table}
In order to compare the results about the \textit{autonomous sensor displacement} in Table \ref{tab:dwasak} for $d=1$ with the \textit{robot assisted displacement} in
Section \ref{sec:line},
we have to reformulate Equation (\ref{aligna1}) in Theorem \ref{theorem:robot} and Corollary \ref{thm:lower} (see Subsection  \ref{sub:line})
for the expected transportation cost.  This is the subject of Theorem \ref{thm:inspiracja01}.
\begin{theorem}
\label{thm:inspiracja01}
Assume that the sensors $X_1\le X_2\le \dots \le X_n$ are initially randomly placed on the $[0,\infty)$ according to Poisson process with arrival rate $\lambda>0.$
Fix $1\le k\le n.$ 
Let $j\in\{0,1,\dots,k-1\}.$
Let $T^{(a,1)}_{(j,RS)}$ be $T^{(a,1)}_{(RS)}$-trans\-por\-tation cost in algorithm $GM_j(k,n).$ 
Then
\begin{equation}
\label{eq:fin200} 
\frac{n^a}{\lambda^a}\left(1+O\left(\frac{1}{n}\right)\right)\le \mathbf{E}\left[\min_{0\le j\le k-1}T^{(a,1)}_{(j,RS)}\right], 
\end{equation}
\begin{equation}
\mathbf{E}\left[\min_{0\le j\le k-1}T^{(a,1)}_{(j,RS)}\right] \le \frac{6^an^a}{\lambda^a}\left(1+O\left(\frac{1}{n}\right)\right)
+\left(4\left\lfloor\frac{n}{k}\right\rfloor+16\right)^{a-1}\Bigg( 4\mathbf{E_{\max}}\left[T^{(a,1)}_{(AS)}\right]
\label{eq:fin200b} +\frac{4\mathbf{E}\left[T^{(a,1)}_{(AS)}\right]}{k}\Bigg). 
\end{equation}
\end{theorem}
\begin{proof}
There are two cases to consider

\textit{Case 1: Inequality (\ref{eq:fin200})}

We know that the random variable $X_i$, i.e., the random position of $i$-th sensor
is the sum of $i$ independent and identically distributed exponential random variables with parameter $\lambda$ and
obeys Gamma distribution with parameters $i\in\mathbb{N}\setminus\{0\},\lambda>0$ 
(see \cite{kingman,ross2002}). Notice that,
\begin{equation}
\label{eq:teta3}
 \mathbf{E}[|X_k|^a]=\frac{1}{\lambda^k}\frac{\Gamma(k+a)}{\Gamma(k)}
 ,\,\,\,\text{when}\,\,\,k\in\{1,2,\dots,n\}\,\,\,\text{and}\,\,\,a>0
\end{equation}
(see \cite[Chapter 15, 22]{stat_disc}).

The following asymptotic expansion for Gamma  function is well known (see \cite[Identity 5.11.13]{NIST} for $z=n$ and $b=0$)
\begin{equation}
\label{eq:gamma_asymema}
\frac{\Gamma(n+a)}{\Gamma(n)}=n^a\left(1+O\left(\frac{1}{n}\right)\right).
\end{equation}

Passing to the expectation in Corollary \ref{thm:lower} and applying Formulas (\ref{eq:teta3}), (\ref{eq:gamma_asymema}) we have the lower bound.

\textit{Case 2: Inequality (\ref{eq:fin200b})}

Passing to the expectation in Equation (\ref{aligna1}) (see Theorem \ref{theorem:robot}), applying Formulas (\ref{eq:icdcna01}), (\ref{eq:icdcna}), (\ref{eq:teta3}), (\ref{eq:gamma_asymema}) as well as the estimation 
$$\mathbf{E}\left[|M_1|^a\right]+\mathbf{E}\left[|M_n|^a\right] \le2\mathbf{E_{\max}}\left[T^{(a,1)}_{(AS)}\right]$$
we get the upper bound. 

It completes the proof of Theorem \ref{thm:inspiracja01}.
\end{proof}
We now compare the new result of Theorem \ref{thm:inspiracja01} for the robot assisted movement ($\min_{0\le j\le k-1}T_{(j,RS)}$ transportation cost) with the known results for the autonomous sensors displacement
($T^{(a,1)}_{(AS)}$ transportation cost) in Table \ref{tab:dwasak} for $d=1.$ Specifically, we find 
$$\min\left(\mathbf{E}\left[\min_{0\le j\le k-1}T^{(a,1)}_{(j,RS)}\right], \mathbf{E}\left[T^{(a,1)}_{(AS)}\right]  \right)$$ 
as the function of the robot capacity $k,$ the sensing radius $r_1$ and the interference distance $s.$

Table \ref{tab:dwas} displays the described comparison of the energy consumption in \textit{the robot assisted movement} and in \textit{the autonomous sensor dis\-place\-ment}.
\begin{table*}[ht]
\caption {Comparison of the expected transportation cost for the
robot assisted movement  with $k$-capacity  and the autonomous sensor displacement to ensure $(r_1,s)$-coverage and interference problem
for $n$ sensors in the $[0,\infty),$ provided that $\epsilon\ge \delta>0$ are arbitrary small constants independent on $\lambda$ and $n.$ } \label{tab:hugher1} 
\begin{center}
\label{tab:dwas}
 \begin{tabular}{|c|c|c|c|} 
 \hline
 \begin{tabular}{c}Sensing \\ radius $r_1,$\end{tabular} & \begin{tabular}{c} Interference \\ distance $s$\end{tabular}  & 
Movement & \begin{tabular}{c}Expected  transportation cost\\(Energy consumption)\end{tabular} \\
\hline
 $r_1=\frac{1+\epsilon}{2\lambda}$ & $s=\frac{1-\delta}{\lambda}$   &  $T^{(a,1)}_{(AS)}$\,\,\, if\,\,\, $a\ge 1$ & $O\left(n\right)/\lambda^a$ \\ 
   \hline
 $r_2=\frac{1}{2\lambda}$ & $s=\frac{1}{\lambda}$   & $T^{(a,1)}_{(AS)}$\,\,\, if\,\,\, $a> 2$  &  $O\left(n^{1+\frac{a}{2}}\right)/\lambda^a$\\
 \hline
    $r_1=\frac{1}{2\lambda}$ & $s=\frac{1}{\lambda}$   & $T^{(a,1)}_{(RS)}$,\,\, if\,\,\, $a\in[1,2]$  &  \begin{tabular}{c}$O\left(n^{\frac{3a}{2}}\right)/(k\lambda)^a\,\,\,$ if $\,\,\,k\in\left[n^{1-\frac{1}{a}}, \sqrt{n}\right]$
 \\ $\,\,\,\,\,\, \Theta\left(n^{a}\right)/\lambda^a\,\,\,$ if $\,\,\,k\in\left[\sqrt{n}, n\right]$\end{tabular}\\
   \hline
 $r_1=\frac{1+\epsilon}{2\lambda}$   & $s=\frac{1+\delta}{\lambda}$       &  $T^{(a,1)}_{(RS)}$\,\,\, if\,\,\, $a\ge 1$ & 
$O\left(n^{2a}\right)/(k\lambda)^a\,\,\,$ if $k\in\left[n^{1-1/a}, n\right]$  \\ 
 \hline
 \end{tabular}
\end{center}
\end{table*}
\subsection{Coverage and Interference on the Plane}
\label{sub:plane_a}
In this subsection we analyze the $(r_2,s)$-coverage and interference problem on the plane. 

We consider $n$ sensors that are randomly placed on the $[0,\infty)\times[0,\infty)$
according to two identical and independent Poisson processes $X_i$ and $Y_i,$
for $i=1,2,\dots, \sqrt{n}$ each with arrival rate $\lambda>0,$ and $n=m^2$ for some $m\in\mathbf{N}.$
The random position of the sensor with identical square sensing radius $r_2$ 
on the plane is determined by the pair $(X_{i_1},Y_{i_2}),$ where $i_{1},i_{2}=1,2,\dots, \sqrt{n}.$ 

Let us recall the concept of \textbf{the square sensing radius.}
\begin{Definition}[cf. \cite{adhocnow2015_KK}, Square Sensing Radius]
\label{def:square}
We assume that a sensor located in position $(x,y)$ can cover any point in the area delimited by the square with $4$ vertices
$(x\pm r_2,y\pm r_2)$ and call $r_2$ the square sensing radius of the sensor.
\end{Definition}
However, if the generally accepted coverage area of a sensor is a circular disk of radius $r_c$ the upper bound results proved in the sequel 
for the square sensing radius $r_2$ hold for circular disk of radius $r_c=\sqrt{2}r_2.$ 
The formal definition of $(r_2,s)$-coverage and interference problem on the plane is as follows.
\begin{Definition}[$(r_2,s)$-coverage and interference]
\label{def:squaresquare}
We move the sensors from initial positions $(X_{i_1},Y_{i_2})$ to the final locations $(X_{i_1}+M_{i_1},Y_{i_2}+N_{i_2}),$ for $i_{1},i_{2}=1,2,\dots, \sqrt{n}$ 
such that:
\begin{itemize}
\item Move the sensors only along the axes provided that 
\begin{align*}
X_1+M_1&\le X_2+M_2\le \dots \le X_{\sqrt{n}}+M_{\sqrt{n}},\\ 
Y_1+N_1&\le  Y_2+N_2\le \dots \le Y_{\sqrt{n}}+N_{\sqrt{n}},
\end{align*}
and the final position of each sensor is in the same row and column as its initial position.
\item Ensure coverage in the sense that every point in the area delimited by the rectangle with $4$ vertices
$(0,0),$ $(X_{\sqrt{n}}+M_{\sqrt{n}},0),$ $(X_{\sqrt{n}}+M_{\sqrt{n}}, Y_{\sqrt{n}}+N_{\sqrt{n}})$,
and $(0,Y_{\sqrt{n}}+N_{\sqrt{n}})$ is in the square sensing radius of at least one sensor.
\item Avoid interference, i.e., each pair of sensors is at interference distance greater or equal to $s$.
\end{itemize}
\end{Definition}
Observe that the main results for the maximum of the expected sensor's displacement metric and $(r_1,s)$-coverage and interference problem on the line
(see \cite{ICDCNkapelko})
are easily applicable to $(r_2,s)$-coverage and interference problem on the plane. For $d=2$, Table \ref{tab:dwasak} displays the results on the plane for
\begin{equation}
\label{eq:icdcna01asek}
\max\left(\mathbf{E}\left[T^{(a,2)}_{(AS)}\right]\right)=\max_{1\le i\le \sqrt{n}}\mathbf{E}\left[|M_i|^a\right]+\max_{1\le i\le \sqrt{n}}\mathbf{E}\left[|N_i|^a\right].
\end{equation}
Moreover, in \cite{ICDCN2020kapelko} the expected transportation cost
\begin{equation}
\label{eq:icdcnabc}
\mathbf{E}\left[T^{(a,2)}_{(AS)}\right]=\sqrt{n}\sum_{i=1}^{\sqrt{n}}\left(\mathbf{E}\left[|M_i|^a\right]+\mathbf{E}\left[|N_i|^a\right]\right)
\end{equation}
was investigated for interference-connectivity problem for the autonomous sensor displacement on the plane.
The main results of \cite{ICDCN2020kapelko} for interference-connectivity problem are applicable to
$(r_2,s)$-coverage and interference problem when the sensors move auto\-no\-mous\-ly on the plane. 
(see  Table \ref{tab:dwasak} for $d=2$).

In order to compare the results about the \textit{autonomous sensor displacement}  in Table \ref{tab:dwasak} for $d=2$ with the \textit{robot assisted displacement} in
Section \ref{sec:plane},
we have to reformulate Theorem \ref{theorem:robot_plane} in Section  \ref{sec:plane}
for the expected transportation cost.  This is the subject of Theorem \ref{thm:inspiracja02}.
\begin{theorem}
\label{thm:inspiracja02}
Assume that $n$ sensors are placed on the plane $[0,\infty)\times[0,\infty)$
according to two identical and independent Poisson processes $X_i$ and $Y_i,$
for $i=1,2,\dots, \sqrt{n}$ each with arrival rate $\lambda>0.$ 
The random position of the sensor on the plane
is determined by the pair $(X_{i_1},Y_{i_2}),$ where $i_{1},i_{2}=1,2,\dots, \sqrt{n}$ and $n=m^2$ for some $m\in\mathbf{N}.$ 
Fix $1\le k\le \sqrt{n}.$
Let $j_1,j_2\in\{0,1,\dots,k-1\}.$
Let $T^{(a,2)}_{(j_1,j_2,RS)}$ be $T^{(a,2)}_{(RS)}$-trans\-por\-tation cost 
in algorithm $GM_{j_1,j_2}(k,n).$ 
Then
\begin{equation}
\label{eq:fin200xy} 
\frac{n^a}{\lambda^a}\left(1+O\left(\frac{1}{\sqrt{n}}\right)\right)\le \mathbf{E}\left[\min_{0\le j_1,j_2\le k-1}T^{(a,2)}_{(j_1,j_2,RS)}\right], 
\end{equation}
\begin{align}
\nonumber \mathbf{E}\left[\min_{0\le j_1,j_2\le k-1}T^{(a,2)}_{(j_1,j_2,RS)}\right] & \le \frac{(2^{3a-1}3^{a})n^a}{\lambda^a}\left(1+O\left(\frac{1}{\sqrt{n}}\right)\right)\\
\label{eq:fin200xyz}&+\left(4\left\lfloor\frac{\sqrt{n}}{k}\right\rfloor+16\right)^{a-1}4^{a-1}n^{a/2}\Bigg( 4\mathbf{E_{\max}}\left[T^{(a,2)}_{(AS)}\right]+\frac{4\mathbf{E}\left[T^{(a,2)}_{(AS)}\right]}{k \sqrt{n}}\Bigg). 
\end{align}
\end{theorem}
\begin{proof}
There are two cases to consider

\textit{Case 1: Inequality (\ref{eq:fin200xy})}

Let us recall that the random variable $X_i-X_1$
is the sum of $i-1$ independent and identically distributed exponential random variables with parameter $\lambda$ and
obeys Gamma distribution with parameters $i-1\in\mathbb{N}\setminus\{0\},\lambda>0$ 
(see \cite{kingman,ross2002}). 
Therefore, passing to the expectation in Corollary \ref{thm:lowerplane} and applying Formulas (\ref{eq:teta3}), (\ref{eq:gamma_asymema}) we have the lower bound.

\textit{Case 2: Inequality (\ref{eq:fin200xyz})}

Passing to the expectation in Equation (\ref{eq:chojnik}) 
(see Theorem \ref{theorem:robot_plane}), applying Formulas (\ref{eq:teta3}), (\ref{eq:gamma_asymema}), (\ref{eq:icdcna01asek}), (\ref{eq:icdcnabc}),  as well as the estimations 
$$\mathbf{E}\left[|M_1|^a\right]+\mathbf{E}\left[|N_1|^a\right]+\mathbf{E}\left[|M_n|^a\right]+\mathbf{E}\left[|N_n|^a\right] \le2\mathbf{E_{\max}}\left[T^{(a,2)}_{(AS)}\right]$$
we get the upper bound. 

It completes the proof of Theorem \ref{thm:inspiracja02}.
\end{proof}
Let us now compare the new result of Theorem \ref{thm:inspiracja02} for the robot assisted movement\\ ($\min_{0\le j_1,j_2\le k-1}T^{(a,2)}_{(j_1,j_2,RS)}$ transportation cost) with the known results for the autonomous sensors displacement
($T^{(a,2)}_{(AS)}$ transportation cost) in Table \ref{tab:dwasak} for $d=2.$ Namely we find 
$$\min\left(\mathbf{E}\left[\min_{0\le j_1,j_2\le k-1}T^{(a,2)}_{(j_1,j_2,RS)}\right], \mathbf{E}\left[T^{(a,2)}_{(AS)}\right]  \right)$$ 
as the function of the robot capacity $k,$ the square sensing radius $r_2$ and the interference distance $s.$

Table \ref{tab:dwas} for $d=2$  displays the described comparison of the energy consumption in \textit{the robot assisted movement} and in \textit{the autonomous sensor displacement}.

\begin{table*}[ht]
\caption {Comparison of the expected transportation cost for the
robot assisted movement  with $k$-capacity  and the autonomous sensor displacement to ensure $(r_2,s)$-coverage and interference problem
for $n$ sensors in the $[0,\infty)^2,$ provided that $\epsilon\ge \delta>0$ are arbitrary small constants independent on $\lambda$ and $n.$ } \label{tab:hugher} 
\begin{center}
\label{tab:dwasasekek}
 \begin{tabular}{|c|c|c|c|} 
 \hline
 \begin{tabular}{c} Square \\sensing \\ radius $r_2$ \end{tabular}& \begin{tabular}{c} Interference \\distance $s$\end{tabular}  & 
Movement &  \begin{tabular}{c}Expected  transportation cost\\(Energy consumption)\end{tabular}\\
\hline
 $r_2=\frac{1+\epsilon}{2\lambda}$ & $s=\frac{1-\delta}{\lambda}$   &  $T^{(a,2)}_{(AS)}$\,\,\, if\,\,\, $a\ge 1$ & $O\left(n\right)/\lambda^a$ \\ 
\hline
 $r_2=\frac{1}{2\lambda}$ & $s=\frac{1}{\lambda}$   & $T^{(a,2)}_{(AS)}$\,\,\, if\,\,\, $a> \frac{4}{3}$  &  $O\left(n^{1+\frac{a}{4}}\right)/\lambda^a$\\
   \hline
    $r_2=\frac{1}{2\lambda}$ & $s=\frac{1}{\lambda}$   & $T^{(a,2)}_{(RS)}$\,\,\, if\,\,\, $a\in\left[1,\frac{4}{3}\right]$  &  \begin{tabular}{c}$O\left(n^{\frac{5a}{4}}\right)/(k\lambda)^a\,\,\,$ if $\,\,\,k\in\left[n^{1-\frac{1}{a}}, n^{\frac{1}{4}}\right]$
 \\ $\,\,\,\,\,\, \Theta\left(n^{a}\right)/\lambda^a\,\,\,$ if $\,\,\,k\in\left[n^{\frac{1}{4}}, \sqrt{n}\right]$\end{tabular}\\
   \hline
 $r_2=\frac{1}{2\lambda}$ & $s=\frac{1}{\lambda}$   & $T^{(a,2)}_{(AS)}$\,\,\, if\,\,\, $a> 2$  & $O\left(n^{1+\frac{a}{2}}\right)/\lambda^a$ \\ 
 \hline

 $r_2=\frac{1+\epsilon}{2\lambda}$  & $s=\frac{1+\delta}{\lambda}$    &  $T^{(a,2)}_{(RS)}$\,\,\, if\,\,\, $a\in\left[1,2\right]$ & 
$O\left(n^{\frac{3a}{2}}\right)/(k\lambda)^a\,\,\,$ if $k\in\left[n^{1-\frac{1}{a}}, \sqrt{n}\right]$  \\ 
 \hline
 \end{tabular}
\end{center}
\end{table*}
\section{Discussions}
\label{sec:discussion}
In this study, the mobile robot with carrying capacity $k$ moves the random sensors on the line and on the plane to provide
the desired scheduling task. The energy consumption in robot assisted movement is compared with the energy consumption in autonomous sensor displacement.
As the application to the coverage and interference requirement, we obtained sharp decrease in the expected transportation cost (see Table \ref{tab:dwasak} and Table \ref{tab:dwas}).

\subsection{Unreliable Sensors}
In the analysis of the coverage and interference requirement, the sensors were assumed to be reliable, i.e., they sense and communicate correctly.
However, the proposed approach in Sections \ref{sec:line} and \ref{sec:plane} is not limited only to the reliable sensors. 
In fact, Theorem \ref{theorem:robot} and Theorem \ref{theorem:robot_plane} also hold for unreliable sensors. 
Moreover, when the sensors can not move, the robot assisted displacement is the only option.

\subsection{Sensors in the Higher Dimensions}
The proposed theory for the robot assisted movement on the plane can be extended to the higher dimensions.
Fix $d\in \mathbf{N}\setminus\{0,1,2\}.$ Assume that $n$ sensors are initially randomly deployed
on the $[0,\infty)^d$ according to $d$ independent and general random processes and $n=m^d$ for
some $m\in\mathbf{N}.$
Let us recall that the mobile robot in Algorithm \ref{robot_assisted_plane} moves the sensors only in vertical
and horizontal fashion. Hence, Algorithm \ref{robot_assisted_plane} can be extended to the sensors on the $[0,\infty)^d$
and the robot's movement to the movement only according to the axes. In particular, the results in Table \ref{tab:dwasasekek} can be extended valid to the sensors  on the $[0,\infty)^d$
\subsection{Many Robots}
Here we discuss a greedy approach to many robots assisted displacement.
Fix $p\in\mathrm{N}.$ Let $k_j\in\mathrm{N}$ and $1\le k_j\le n$ for $j=1,2,\dots, p$. Assume that $j$-th robot can carry at most $k_j$ sensors 
and deposit all or part of them at any time and to any suitable position it chooses.
The carrying capacity  of the system of $p$ robots is defined as $\sum_{j=1}^{p}k_j.$
Algorithm \ref{alg_anchor} (Greedy Procedure) can be modified  for the system
of $p$ robots. Let $k:=\sum_{j=1}^{p}k_j$ and $l\le \sum_{j=1}^{p}k_j$ in Algorithm \ref{alg_anchor}. Assume that all $p$ robots move
together, collect the sensors $Y_1\le Y_2\le\dots \le Y_l$ and displace at the final positions. Then,
it is possible to develop the theory similar to that presented in this paper. In particular, in modified Table \ref{tab:dwas} and Table \ref{tab:dwasasekek} for the system of $p$ robots, the parameter $k$ is the arithmetic mean
of capacities $k_1,k_2,\dots, k_p$, i.e., $k=\frac{1}{p}\sum_{j=1}^{p}k_j$.
%
%
\section{Numerical Results}
\label{sec:experiments}
In this section we use a set of experiments to show how the robot capacity $k$ impacts the transportation cost for the coverage and interference
requirement when the $n$ sensors are randomly placed on the line according to Poisson process with arrival rate $\lambda=n$ and on the plane according 
to two identical and independent Poisson processes each with arrival rate $\lambda=\sqrt{n}.$
Namely, we implement Algorithm \ref{robot_assisted} for $a=1,$ $j=0$ and $k\in\left\{1,\lceil\sqrt{n}\rceil, n\right\};$ 
Algorithm \ref{robot_assisted} for $a=2,$ $j=0$ and $k\in\left\{\lceil\sqrt{n}\rceil, n\right\};$
Algorithm \ref{robot_assisted_plane} for $a=1,$ $j_1=j_2=0$ and $k\in\left\{1,\sqrt{n}\right\}.$

\subsection{Evaluation of Algorithm \ref{robot_assisted} for $a=1$ and $j=0$}
\begin{table*}
\begin{figure}[H]
  \begin{center}
\begin{tabular}{cc}
 \begin{minipage}[b]{0.45\textwidth} \centering \includegraphics[width=1.00\textwidth]{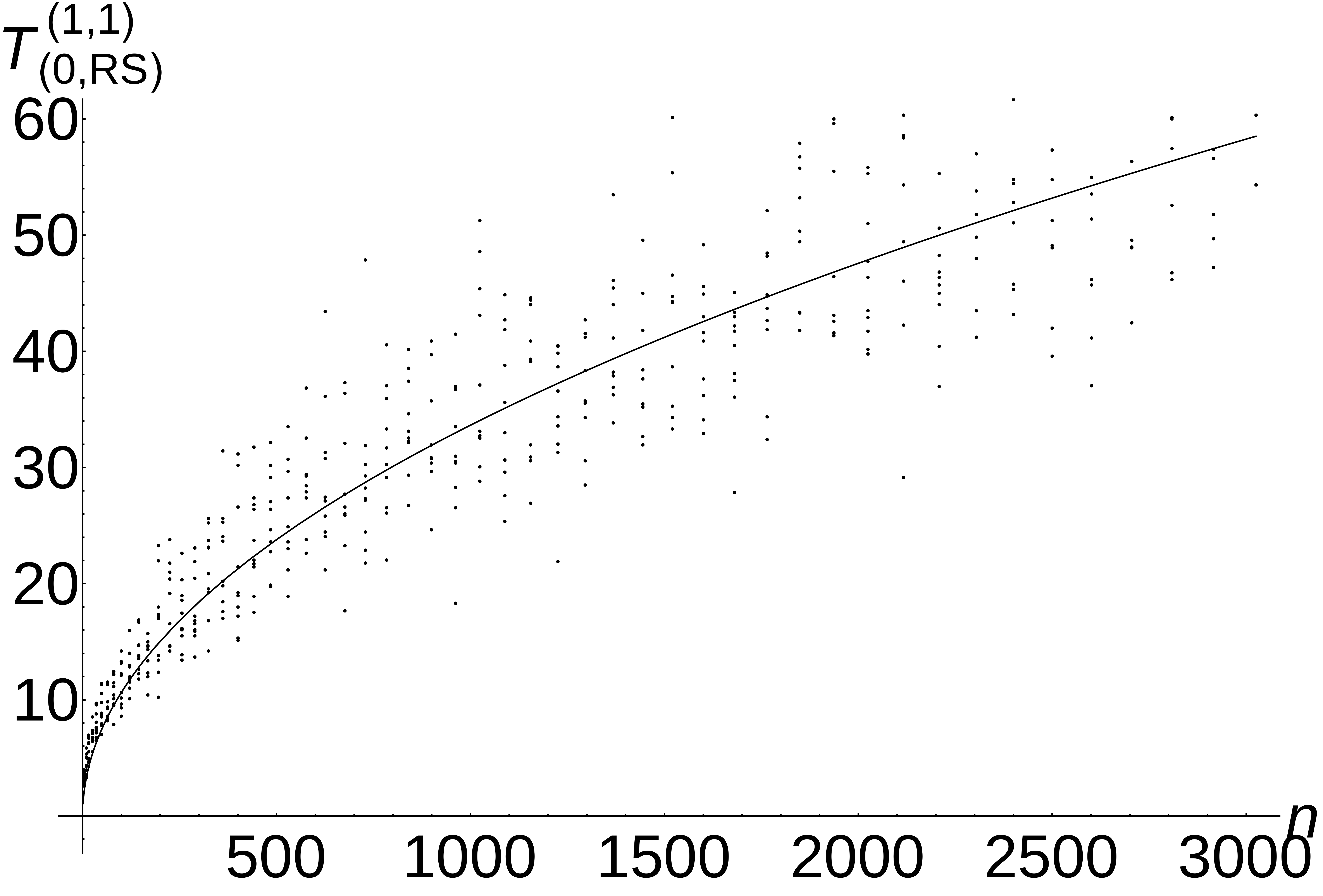}\\  $T^{(1,1)}_{(0,RS)}=\frac{\sqrt{2}}{\Gamma\left(\frac{5}{2}\right)}\sqrt{n}$
 \vspace{-5pt}
  \caption{The expected $T^{(1,1)}_{(0,RS)}$ transportation cost of Algorithm $GM_{0}(1,n)$ for $M_i=\frac{i}{n}-\frac{1}{2n}$ where $i=1,2,\dots,n$ and $k=1$ }\label{b:1}
 \label{fig.one}
 \end{minipage}
 &
 \begin{minipage}[b]{0.45\textwidth} \centering \includegraphics[width=1.00\textwidth]{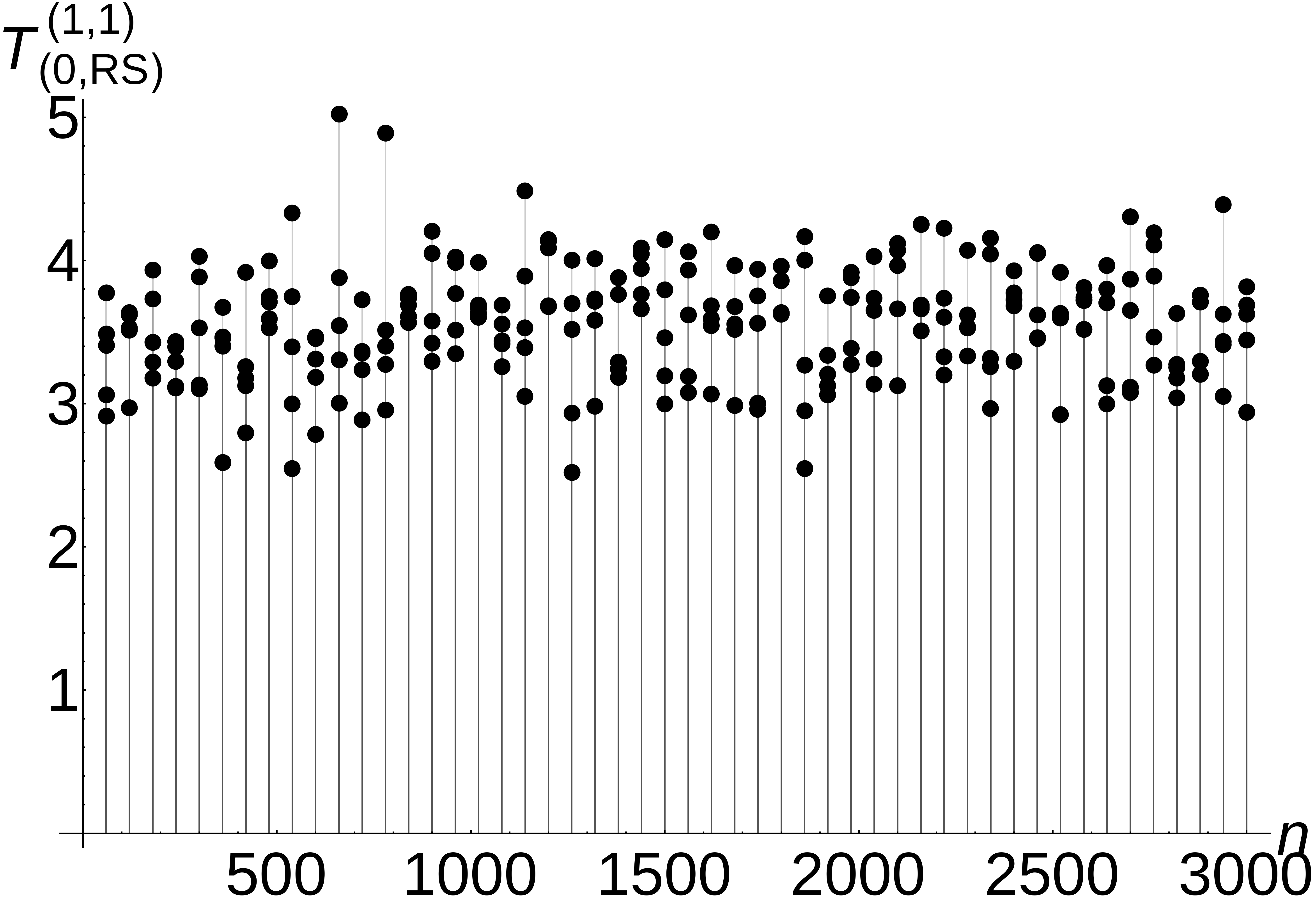}\\   $T^{(1,1)}_{(0,RS)}=\Theta(1)$
 \vspace{-5pt}
   \caption{The expected $T^{(1,1)}_{(0,RS)}$ transportation cost of Algorithm $GM_{0}\left(\lceil\sqrt{n}\rceil,n\right)$ for $X_i+M_i=\frac{i}{n}-\frac{1}{2n},$ where $i=1,2,\dots,n$ and $k=\lceil\sqrt{n}\rceil$}\label{b:2}
  \label{fig.two}
 \end{minipage}
  \\
\end{tabular}
\end{center}
\end{figure}
\end{table*}
\begin{table*}
\begin{figure}[H]
  \begin{center}
\begin{tabular}{cc}
 \begin{minipage}[b]{0.45\textwidth} \centering \includegraphics[width=1.00\textwidth]{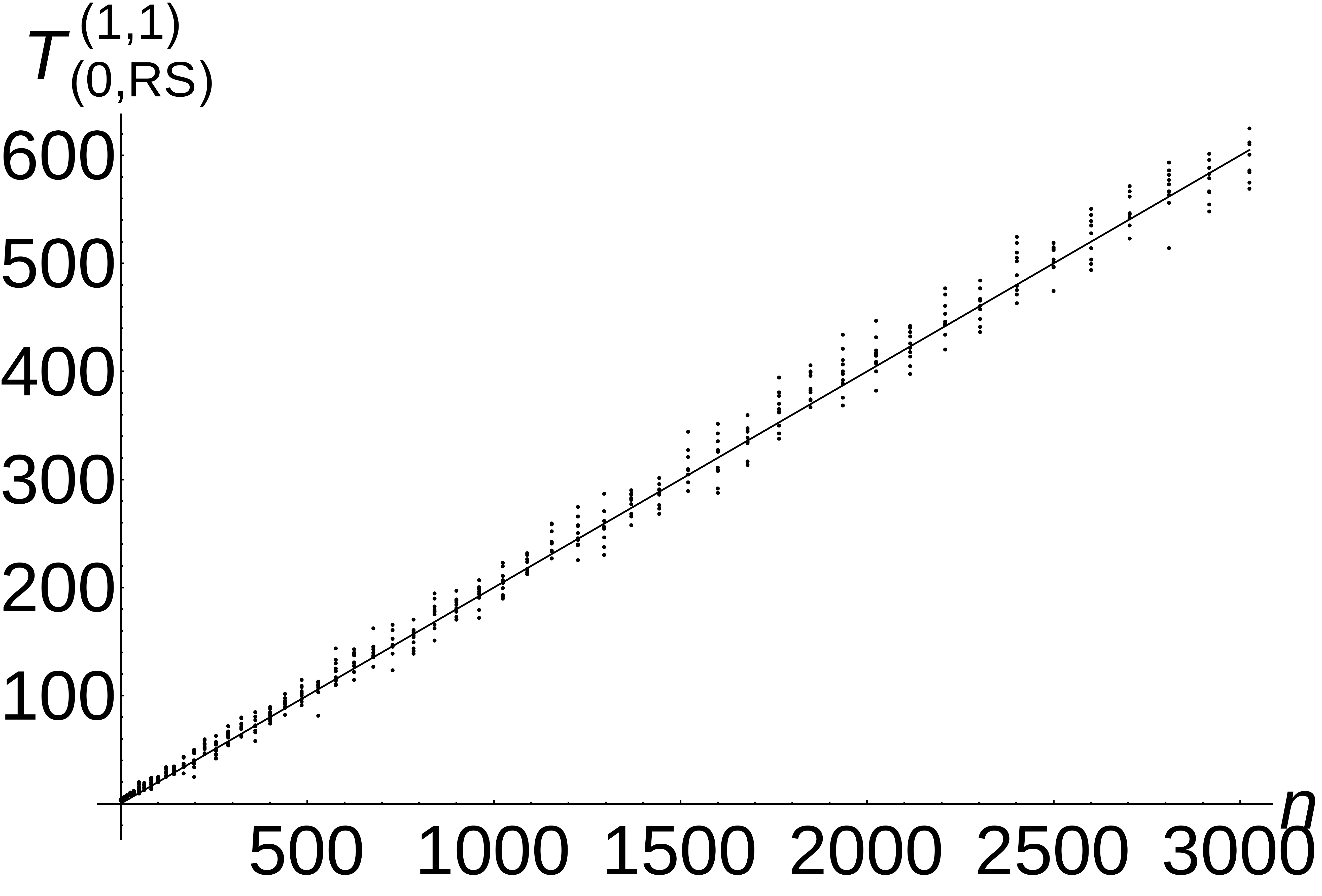}\\  $T^{(1,1)}_{(0,RS)}=\frac{1}{5}n$
 \vspace{-5pt}
   \caption{The expected $T^{(1,1)}_{(0,RS)}$ transportation cost of Algorithm $GM_{0}(1,n)$ for $X_i+M_i=\left(1+\frac{1}{5}\right)\left(\frac{i}{n}-\frac{1}{2n}\right),$ where $i=1,2,\dots,n$ and $k=1$ }\label{b:3}
 \label{fig.three}
 \end{minipage}
 &
 \begin{minipage}[b]{0.45\textwidth} \centering \includegraphics[width=1.00\textwidth]{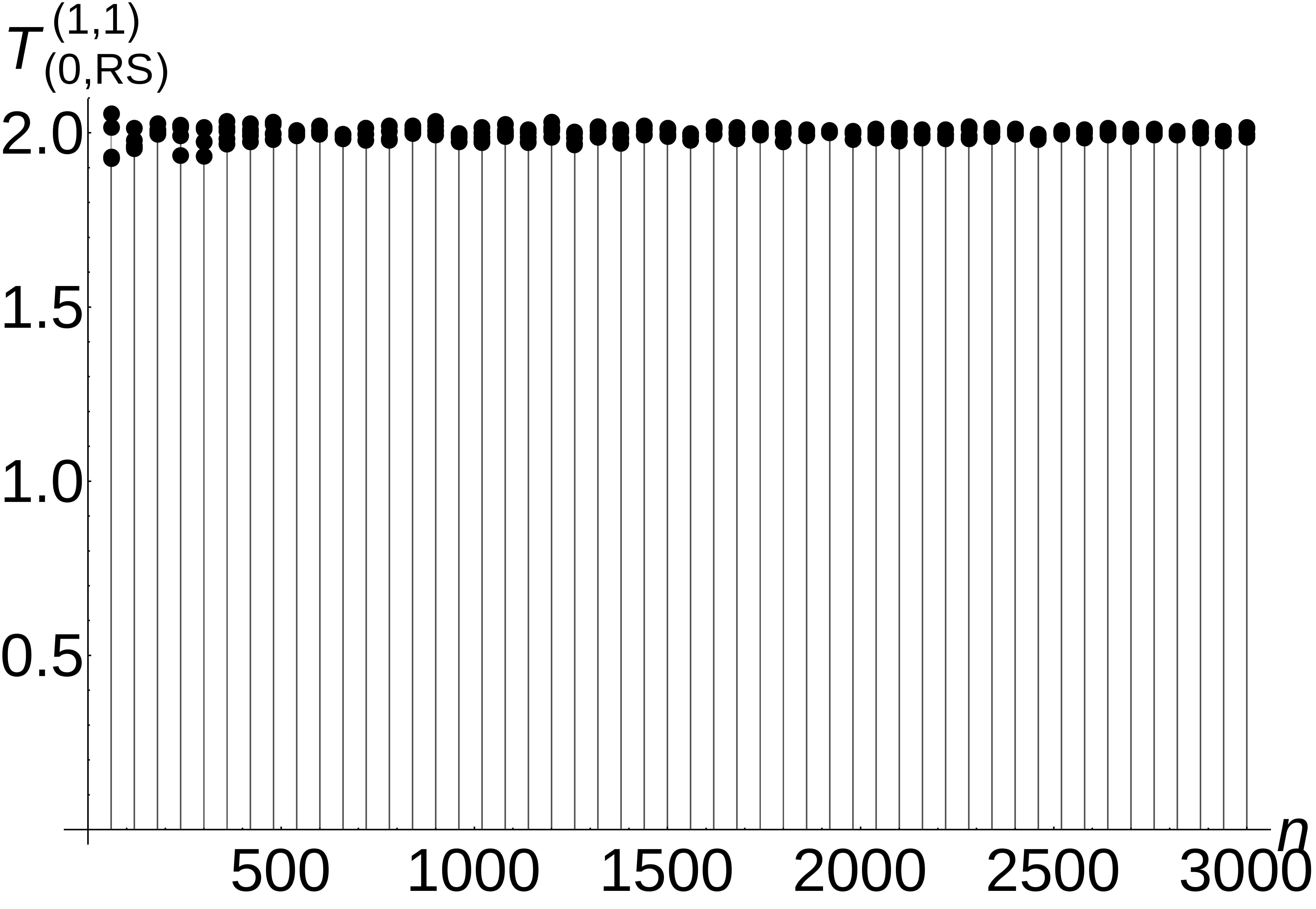}\\   $T^{(1,1)}_{(0,RS)}=\Theta(1)$
 \vspace{-5pt}
    \caption{The expected $T^{(1,1)}_{(0,RS)}$ transportation cost of Algorithm $GM_{0}(n,n)$ for $X_i+M_i=\left(1+\frac{1}{5}\right)\left(\frac{i}{n}-\frac{1}{2n}\right),$ where $i=1,2,\dots,n$ and $k=n$ }\label{b:4} 
  \label{fig.four}
 \end{minipage}
  \\
\end{tabular}
\end{center}
\end{figure}
\end{table*}

For the case of Algorithm $GM_0(1,n)$ when the robot capacity $k=1$ (Algorithm \ref{robot_assisted} for $j=0$ and $k=1$) we conduct Algorithm \ref{a3}.
\begin{algorithm}[H]
\caption{}
\label{a3}
\begin{algorithmic}[1]
\STATE{$n:=1$}
\WHILE{n $\le$ 55}
 \FOR{$l=1$  \TO $100$ } 
 \STATE{Generate $X_1(l), X_2(l), \dots, X_{n^2}(l)$ random points on the $[0,\infty)$ such that $\forall_{i\in\{1,2,\dots, n^2 \}}\,\, X_i(l)$ is the sum of 
 $i$ independent and identically distributed exponential random variables with parameter $\lambda=n^2;$}
 \STATE{Calculate $T_{1,n^2}(l)$ the distance travelled by the robot\\ in Algorithm \ref{robot_assisted} for $j=0,$ and $k=1;$}
 \ENDFOR
  \FOR{$p=1$  \TO $10$ } 
 \STATE{Calculate  the average\\ $T_{1,n^2}=\frac{1}{5}\sum_{q=1}^{5} T_{1,n^2}(q+(p-1)*10);$}
 \STATE{Insert the points $(n^2, T_{1,n^2})$ into the chart;}
 \ENDFOR
 \STATE{$n:=n+1$}
 \ENDWHILE
\end{algorithmic}
\end{algorithm}

In Figures \ref{b:1} and \ref{b:3} the black dots represent the numerical results of conducted Algorithm \ref{a3} for $X_i+M_i=\frac{i}{n}-\frac{1}{2n}$
and $X_i+M_i=\left(1+\frac{1}{5}\right)\left(\frac{i}{n}-\frac{1}{2n}\right)$ where $i=1,2,\dots, n.$ The additional lines\\ $\left\{(n,\frac{\sqrt{2}}{\Gamma(5/2)}\sqrt{n}), 1\le n\le 55^2\right\},$
$\left\{(n,\frac{1}{5}n), 1\le n\le 55^2\right\},$
are the leading terms in the theoretical estimation of the distance travelled by the robot
in Algorithm $GM_0(1,n)$ for $k=1$ when  $X_i+M_i=\frac{i}{n}-\frac{1}{2n}$
and $X_i+M_i=\left(1+\frac{1}{5}\right)\left(\frac{i}{n}-\frac{1}{2n}\right)$ provided $i=1,2,\dots, n.$

\begin{algorithm}[H]
\caption{}
\label{a4}
\begin{algorithmic}[1]
\STATE{$n:=1$}
\WHILE{n $\le$ 50}
 \FOR{$l=1$  \TO $50$ } 
 \STATE{Generate $X_1(l), X_2(l),\dots,X_{60n}(l)$ random points on the $[0,\infty)$ such that $\forall_{i\in\{1,2,\dots, n^2 \}}\,\, X_i(l)$ is the sum of 
 $i$ independent and identically distributed exponential random variables with parameter $\lambda=n^2;$}

 \STATE{Calculate $\mathbf{V_{a, 60n}}(l)$ the distance to the power $a$ travelled by the robot in Algorithm \ref{robot_assisted} for $j=0,$ and $k;$}
 \ENDFOR
  \FOR{$p=1$  \TO $5$ } 
 \STATE{Calculate  the average\\  $\mathbf{V_{a, 60n}}=\frac{1}{10}\sum_{v=1}^{10} \mathbf{V_{a, 60n}}(v+(p-1)10);$}
 \STATE{Insert the points $\mathbf{V_{a}}(60n)$ into the chart;}
 \ENDFOR
 \STATE{$n:=n+1$}
 \ENDWHILE
\end{algorithmic}
\end{algorithm}

Figures \ref{b:2} and \ref{b:4} depict the experimental distance travelled by the robot in Algorithm \ref{robot_assisted} considering the parameters
$j=0,$ $k=\lceil\sqrt{n}\rceil,$ $X_i+M_i=\frac{i}{n}-\frac{1}{2n},$  and $j=0,$ $k=n,$ $X_i+M_i=\left(1+\frac{1}{5}\right)\left(\frac{i}{n}-\frac{1}{2n}\right).$
In this case we conduct Algorithm \ref{a4} considering the parameters $a=1,$ $k=\lceil\sqrt{n}\rceil;$ $a=1,$ $k=n.$ 

The experimental distance travelled by the robot in Figures \ref{b:2} and \ref{b:4} is in $\Theta(1).$ Hence, the carried out experiments
confirm very well the obtained theoretical upper bound $O(1)$ (see Table \ref{tab:dwas} for $a=1,$ $\lambda=n,$ $r_1=\frac{1}{2n},$ $s=\frac{1}{n},$
$k=\lceil\sqrt{n}\rceil$ and for $a=1,$ $\lambda=n,$ $r_1=\frac{1+\frac{1}{5}}{2n},$ $s=\frac{1+\frac{1}{5}}{n},$
$k=n$).

\subsection{Evaluation of Algorithm \ref{robot_assisted} for $a=2$ and $j=0$}
\begin{table*}
\begin{figure}[H]
  \begin{center}
\begin{tabular}{cc}
 \begin{minipage}[b]{0.45\textwidth} \centering \includegraphics[width=1.00\textwidth]{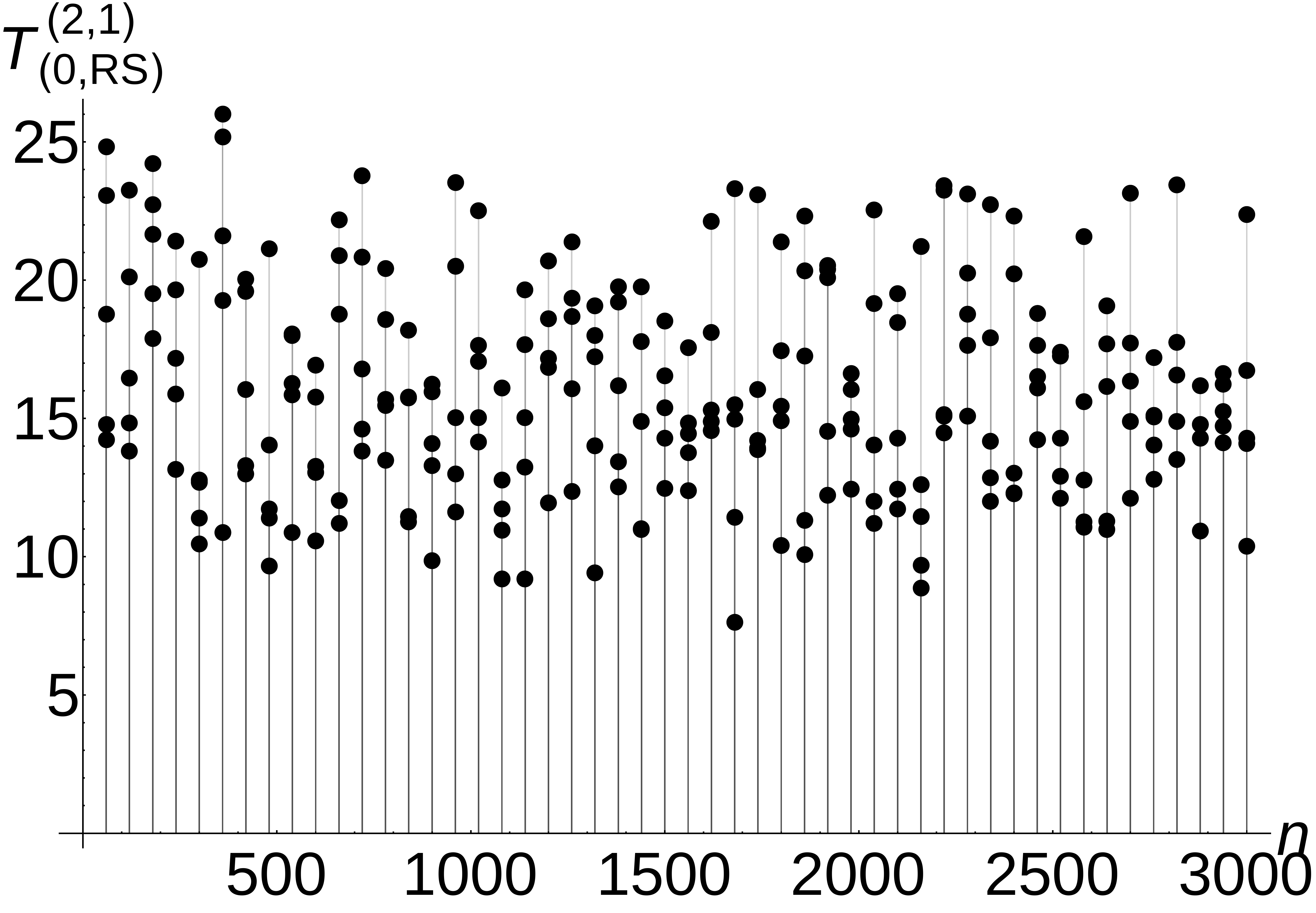}\\  $T^{(2,1)}_{(0,RS)}=\Theta(1)$
 \vspace{-5pt}
 \caption{The expected $T^{(2,1)}_{(0,RS)}$ transportation cost of Algorithm $GM_{0}\left(\lceil\sqrt{n}\rceil,n\right)$ for $X_i+M_i=\frac{i}{n}-\frac{1}{2n},$ where $i=1,2,\dots,n$ and $k=\lceil\sqrt{n}\rceil$}\label{b:3fgh}
 \label{fig.two_power}
 \end{minipage}
 &
 \begin{minipage}[b]{0.45\textwidth} \centering \includegraphics[width=1.00\textwidth]{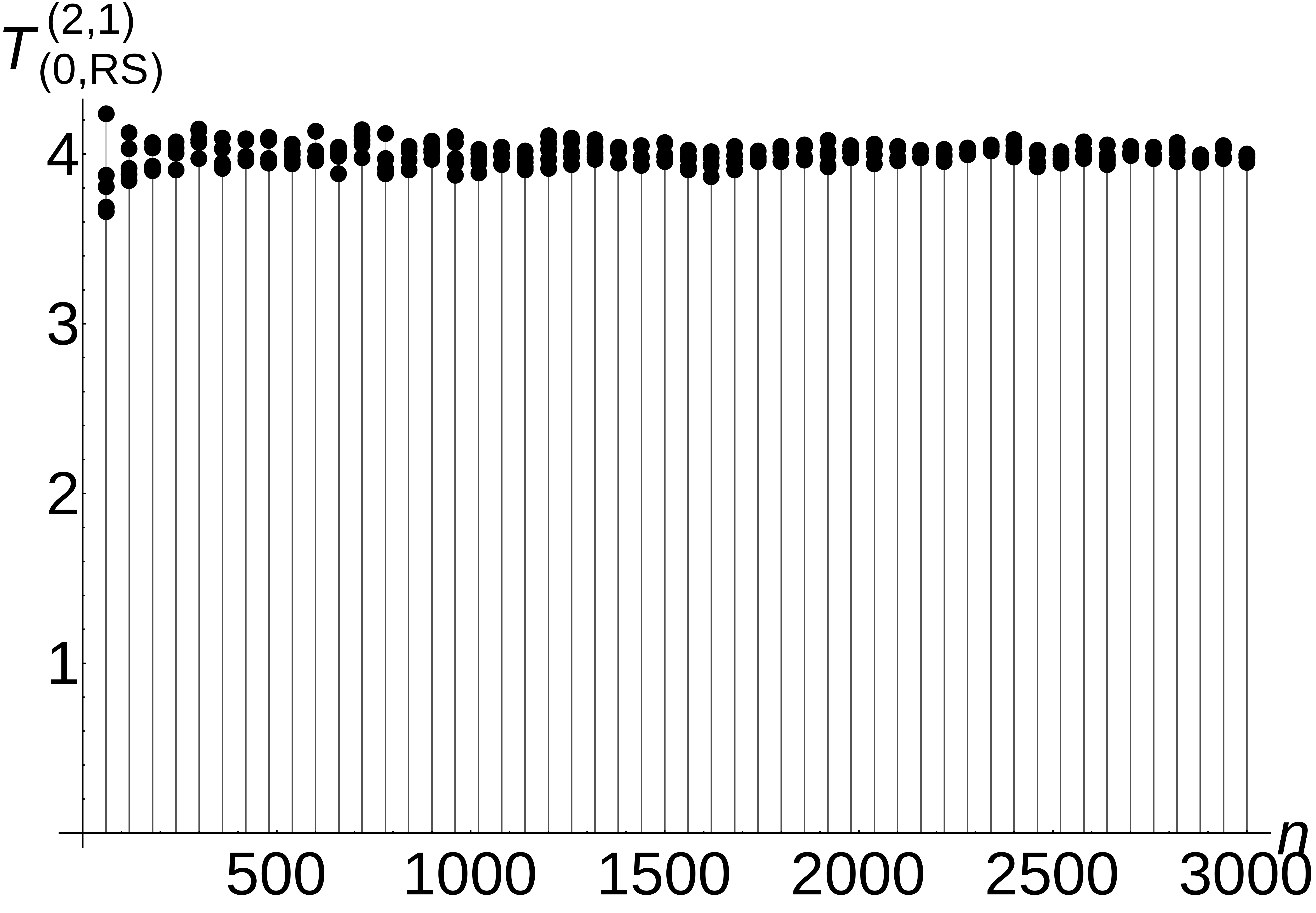}\\   $T^{(2,1)}_{(0,RS)}=\Theta(1)$
 \vspace{-5pt}
   \caption{The expected $T^{(2,1)}_{(0,RS)}$ transportation cost of Algorithm $GM_{0}(1,n)$ for $X_i+M_i=\left(1+\frac{1}{5}\right)\left(\frac{i}{n}-\frac{1}{2n}\right),$ where $i=1,2,\dots,n$ and $k=n$ }\label{b:3fd}
  \label{fig.four_power}
 \end{minipage}
  \\
\end{tabular}
\end{center}
\end{figure}
\end{table*}
We evaluate the distance to the power $2$ travelled by the robot in Algorithm $GM_{0}(k,n)$ when the robot capacity $k=\lceil\sqrt{n}\rceil$ or $k=n$ (Algorithm \ref{robot_assisted}
for $j=0$ and $k=\lceil\sqrt{n}\rceil,  n$).

Figures \ref{b:3fgh} and \ref{b:3fd} depict the experimental distance to the power $2$ travelled by the robot in Algorithm \ref{robot_assisted} considering the parameters
$j=0,$ $k=\lceil\sqrt{n}\rceil,$ $X_i+M_i=\frac{i}{n}-\frac{1}{2n},$  and $j=0,$ $k=n,$ $X_i+M_i=\left(1+\frac{1}{5}\right)\left(\frac{i}{n}-\frac{1}{2n}\right).$
In this case we conduct Algorithm \ref{a4} considering the parameters $a=2,$ $k=\lceil\sqrt{n}\rceil;$ $a=2,$ $k=n.$ 

The experimental distance to the power $2$ travelled by the robot in Figures \ref{b:3fgh} and \ref{b:3fd} is in $\Theta(1).$ Hence, the carried out experiments
confirm very well the obtained theoretical upper bound $O(1)$ (see Table \ref{tab:dwas} for $a=2,$ $\lambda=n,$ $r_1=\frac{1}{2n},$ $s=\frac{1}{n},$
$k=1,\lceil\sqrt{n}\rceil$ and for $a=2,$ $\lambda=n,$ $r_1=\frac{1+\frac{1}{5}}{2n},$ $s=\frac{1+\frac{1}{5}}{n},$

\subsection{Evaluation of Algorithm \ref{robot_assisted_plane} for $a=1$ and $j_1=j_2=0$}
For the case of Algorithm $GM_{0,0}(k,n)$ when the robot capacity $k=1$ or $k=\sqrt{n}$ (Algorithm \ref{robot_assisted_plane}
for $j_1=j_2=0$ and $k=1,\sqrt{n}$),
we conduct Algorithm \ref{a5} for $k=1$ and $k=\sqrt{n}.$
\vspace{3cm}
\begin{algorithm}[H]
\caption{}
\label{a5}
\begin{algorithmic}[1]
\STATE{$n:=1$}
\WHILE{n $\le$ 55}
 \FOR{$l=1$  \TO $100$ } 
 \STATE{Generate $n^2$ random points on the $[0,\infty)\times[0,\infty)$ 
 according to two identical and independent Poisson processes $X_i$ and $Y_i,$
for $i=1,2,\dots, \sqrt{n^2}$ each with arrival rate $\lambda=\sqrt{n^2}.$ 
The random position of the sensor with identical \textit{square sensing radius} $r_2$ 
on the plane is determined by the pair $(X_{i_1}(l),Y_{i_2}(l)),$ where $i_{1},i_{2}=1,2,\dots \sqrt{n};$} 
 \STATE{Calculate $W_{1,n^2}(l)$ the distance travelled by the robot\\ in Algorithm \ref{robot_assisted_plane} for $j_1=j_2=0$ and $k;$}
 \ENDFOR
  \FOR{$p=1$  \TO $10$ } 
 \STATE{Calculate  the average\\ $W_{1,n^2}=\frac{1}{5}\sum_{q=1}^{5} W_{1,n^2}(q+(p-1)*10);$}
 \STATE{Insert the points $(n^2, W_{1,n^2})$ into the chart;}
 \ENDFOR
 \STATE{$n:=n+1$}
 \ENDWHILE
\end{algorithmic}
\end{algorithm}

\begin{table*}
\begin{figure}[H]
  \begin{center}
\begin{tabular}{cc}
 \begin{minipage}[b]{0.45\textwidth} \centering \includegraphics[width=1.00\textwidth]{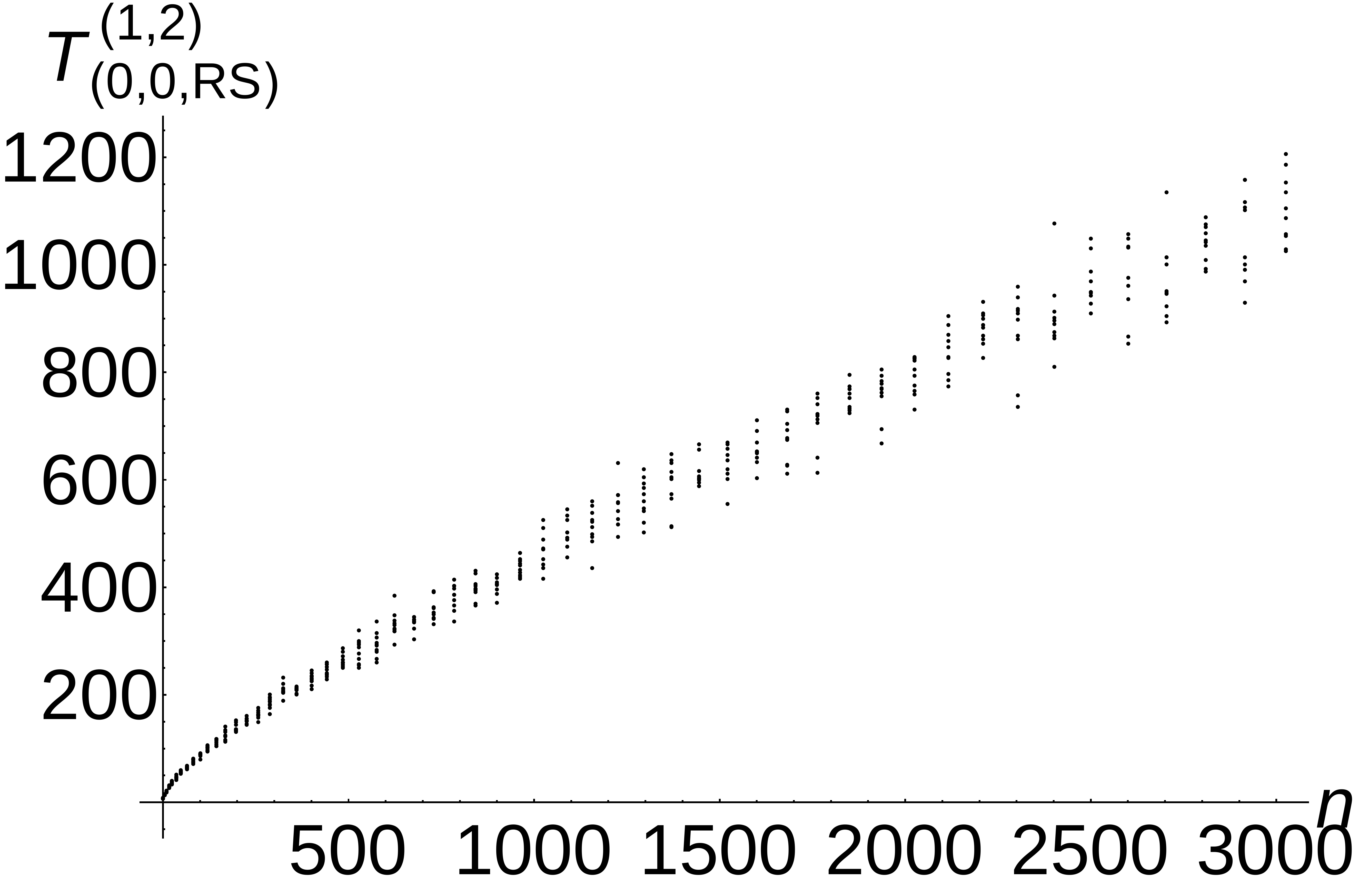}\\   $T^{(1,2)}_{(0,0,RS)}=\Theta(n)$
 \vspace{-5pt}
   \caption{The expected $T^{(1,2)}_{(0,0,RS)}$ transportation cost of Algorithm $GM_{0,0}(1,n)$ for 
  $\left(X_{i_1}+M_{i_1},Y_{i_2}+M_{i_2}\right)=\left(1+\frac{1}{5}\right)\left(\frac{i_1}{\sqrt{n}}-\frac{1}{2\sqrt{n}},\frac{i_2}{\sqrt{n}}-\frac{1}{2\sqrt{n}}\right),$ where $i_1, i_2=1,2,\dots,\sqrt{n}$ and $k=1$ }\label{b:5}
 \label{fig.five}
 \end{minipage}
 &
 \begin{minipage}[b]{0.45\textwidth} \centering \includegraphics[width=1.00\textwidth]{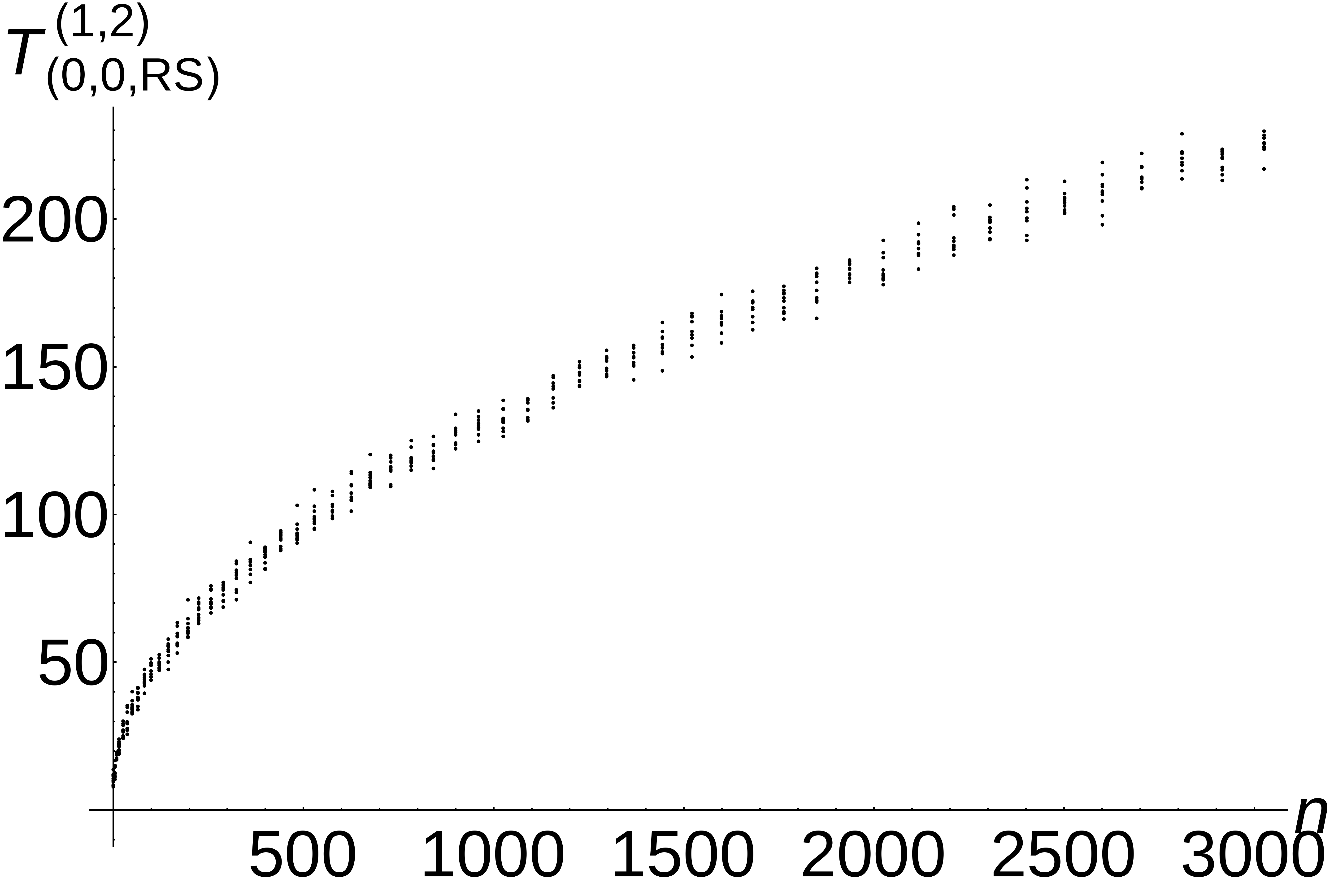}\\   $T^{(1,2)}_{(0,0,RS)}=\Theta(\sqrt{n})$
 \vspace{-5pt}
  \caption{The expected $T^{(1,2)}_{(0,0,RS)}$ transportation cost of Algorithm $GM_{0,0}(\sqrt{n},n)$ for 
  $\left(X_{i_1}+M_{i_1},Y_{i_2}+M_{i_2}\right)=\left(1+\frac{1}{5}\right)\left(\frac{i_1}{\sqrt{n}}-\frac{1}{2\sqrt{n}},\frac{i_2}{\sqrt{n}}-\frac{1}{2\sqrt{n}}\right),$ where $i_1, i_2=1,2,\dots,\sqrt{n}$ and $k=\sqrt{n}$ }\label{b:6}  
  \label{fig.six}
 \end{minipage}
  \\
\end{tabular}
\end{center}
\end{figure}
\end{table*}

In Figures \ref{b:4} and \ref{b:5}, the black dots represent the numerical results of conducted  Algorithm \ref{a5}
for\\
\centerline{$\left(X_{i_1}+M_{i_1},Y_{i_2}+M_{i_2}\right)=\left(1+\frac{1}{5}\right)\left(\frac{i_1}{\sqrt{n}}-\frac{1}{2\sqrt{n}},\frac{i_2}{\sqrt{n}}-\frac{1}{2\sqrt{n}}\right)$}\\
\centerline{where $i_1, i_2=1,2,\dots,\sqrt{n},$ $k=1$ and $k=\sqrt{n}.$}

Notice that, the experimental distance travelled by the robot in Figure \ref{b:4} is in $\Theta(n)$ and in Figure \ref{b:5} is in $\Theta(\sqrt{n}).$ Hence, the carried out experiments
confirm very well the obtained theoretical upper bound $O(n)$ and $O(\sqrt{n})$ (see Table \ref{tab:dwasasekek} for $a=1,$ $\lambda=\sqrt{n},$ $r_2=\frac{1+\frac{1}{5}}{2\sqrt{n}},$ $s=\frac{1+\frac{1}{5}}{\sqrt{n}},$
$k=1$ and for $a=1,$ $\lambda=\sqrt{n},$ $r_2=\frac{1+\frac{1}{5}}{2\sqrt{n}},$ $s=\frac{1+\frac{1}{5}}{\sqrt{n}},$
$k=\sqrt{n}$).

\section{Conclusions}
In this paper we addressed the problem of the robot assisted movement in WSNs in which the sensors are initially
randomly deployed on the line according to general random process, and on the plane according to two independent general random processes.
To this end, we compared the transportation cost for robot assisted movement with the transportation cost for autonomous sensor displacement.
We obtained tradeoffs between the energy consumption in robot's movement, the number of sensors $n$, the sensor range $r$, interference
distance $s$, the robot capacity $k$ until completion of the
coverage and interference scheduling task.
While we have discussed the applicability of our approach to greedy many robots assisted displacement, an open problem for future study is
the system of many robots assisted movement. Additionally, it would be interesting to explore the system of many robots for others trajectories on the plane and in the higher dimension.
\bibliographystyle{plain}
\bibliography{refs}

\begin{thebibliography}{10}

\bibitem{abbasi2009movement}
A.A. Abbasi, M.~Younis, and K.~Akkaya.
\newblock Movement-assisted connectivity restoration in wireless sensor and
  actor networks.
\newblock {\em Parallel and Distributed Systems, IEEE Trans. on},
  20(9):1366--1379, 2009.

\bibitem{percolation}
H.~M. Ammari and S.~K. Das.
\newblock Critical density for coverage and connectivity in three-dimensional
  wireless sensor networks using continuum percolation.
\newblock {\em IEEE Transactions on Parallel and Distributed Systems},
  20(6):872--885, 2009.

\bibitem{faultdas}
H.~M. Ammari and S.~K. Das.
\newblock Fault tolerance measures for large-scale wireless sensor networks.
\newblock {\em ACM Trans. Auton. Adapt. Syst.}, 4(1), 2009.

\bibitem{salajkcover}
H.~M. Ammari and S.~K. Das.
\newblock Centralized and clustered k-coverage protocols for wireless sensor
  networks.
\newblock {\em IEEE Transactions on Computers}, 61(1):118--133, 2012.

\bibitem{tcs2009}
B.~Bhattacharya, M.~Burmester, Y.~Hu, E.~Kranakis, Q.~Shi, and A.~Wiese.
\newblock Optimal movement of mobile sensors for barrier coverage of a planar
  region.
\newblock {\em TCS}, 410(52):5515--5528, 2009.

\bibitem{slam}
Y.~Chen, S.~Huang, and R.~Fitch.
\newblock Active slam for mobile robots with area coverage and obstacle
  avoidance.
\newblock {\em IEEE/ASME Transactions on Mechatronics}, 25(3):1182--1192, 2020.

\bibitem{ptas}
Z.~Chen, X.~Gao, F.~Wu, and G.~Chen.
\newblock A ptas to minimize mobile sensor movement for target coverage
  problem.
\newblock In {\em IEEE INFOCOM 2016 - The 35th Annual IEEE International
  Conference on Computer Communications}, pages 1--9, 2016.

\bibitem{robot2020}
J.~Czyzowicz, E.~Kranakis, D.~Krizanc, L.~Narayanan, and J.~Opatrny.
\newblock Optimal online and offline algorithms for robot-assisted restoration
  of barrier coverage.
\newblock {\em Discrete Applied Mathematics}, 285:650 -- 662, 2020.

\bibitem{transaction2021}
Sajal~K. Das and R.~Kapelko.
\newblock On the range assignment in wireless sensor networks for minimizing
  the coverage-connectivity cost.
\newblock {\em ACM Trans. Sen. Netw.}, 17(4), 2021.

\bibitem{Dobrev2017}
S.~Dobrev, E.~Kranakis, D.~Krizanc, M.~Lafond, J.~Manuch, L.~Narayanan,
  J.~Opatrny, S.~Shende, and L.~Stacho.
\newblock Weak coverage of a rectangular barrier.
\newblock {\em Algorithmica}, 82:721--746, 2020.

\bibitem{kranakis_shaikhet}
Kranakis E. and Shaikhet G.
\newblock Displacing random sensors to avoid interference.
\newblock In {\em COCOON}, volume 8591 of {\em LNCS}, pages 501--512. Springer,
  2014.

\bibitem{stat_disc}
C.~Forbes, M.~Evans, N.~Hasting, and B.~Peacock.
\newblock {\em Statistical Distributions}.
\newblock Wiley, 2011.

\bibitem{siamcontrol2015}
P.~Frasca, F.~Garin, B.~Gerencs\'er, and Hendrickx~J. M.
\newblock Optimal one-dimensional coverage by unreliable sensors.
\newblock {\em SIAM Journal on Control and Optimization}, 53(5):3120--3140,
  2015.

\bibitem{fuchs2020}
M.~Fuchs, L.~Kao, and W.~Wu.
\newblock On binomial and poisson sums arising from the displacement of
  randomly placed sensors.
\newblock {\em Taiwanese Journal of Mathematics}, 24(6):1353 -- 1382, 2020.

\bibitem{trans2018}
X.~Gao, J.~Fan, F.~Wu, and G.~Chen.
\newblock Approximation algorithms for sweep coverage problem with multiple
  mobile sensors.
\newblock {\em IEEE/ACM Transactions on Networking}, 26(2):990--1003, 2018.

\bibitem{sajal2008}
A.~Ghosh and S.~K. Das.
\newblock Coverage and connectivity issues in wireless sensor networks: A
  survey.
\newblock {\em Pervasive and Mobile Computing}, 4:303--334, 2008.

\bibitem{ICDCNkapelko}
R.~Kapelko.
\newblock On the maximum movement of random sensors for coverage and
  interference on a line.
\newblock In {\em Proceedings of the 19th International Conference on
  Distributed Computing and Networking}, pages 36:1--36:10. ACM, 2018.

\bibitem{pervasiveKAPELKO}
R.~Kapelko.
\newblock On the maximum movement to the power of random sensors for coverage
  and interference.
\newblock {\em Pervasive and Mobile Computing}, 51:174 -- 192, 2018.

\bibitem{kapelko_dmaa}
R.~Kapelko.
\newblock Asymptotic formula for sum of moment mean deviation for order
  statistics from uniform distribution.
\newblock {\em Discrete Mathematics, Algorithms and Applications}, 11:1--23,
  April 2019.

\bibitem{ICDCN2020kapelko}
R.~Kapelko.
\newblock On the energy in displacement of random sensors for interference and
  connectivity.
\newblock In {\em Proceedings of the 21st International Conference on
  Distributed Computing and Networking}, pages 1--10. ACM, 2020.

\bibitem{adhocnow2015_KK}
R.~Kapelko and E.~Kranakis.
\newblock On the displacement for covering a square with randomly placed
  sensors.
\newblock In {\em ADHOCNOW}, volume 9143 of {\em LNCS}, pages 148--162.
  Springer, 2015.

\bibitem{KK_2016_cube}
R.~Kapelko and E.~Kranakis.
\newblock On the displacement for covering a d-dimensional cube with randomly
  placed sensors.
\newblock {\em Ad Hoc Networks}, 40:37--45, 2016.

\bibitem{kapelkokranakisIPL}
R.~Kapelko and E.~Kranakis.
\newblock On the displacement for covering a unit interval with randomly placed
  sensors.
\newblock {\em Information Processing Letters}, 116:710--717, 2016.

\bibitem{khan_ama}
A.~Khan, I.~Noreen, and Z.~Habib.
\newblock An energy efficient coverage path planning approach for mobile
  robots.
\newblock In {\em Intelligent Computing}, pages 387--397. Springer
  International Publishing, 2019.

\bibitem{kim2017}
H.~Kim, H.~Oh, P.~Bellavista, and J.~Ben-Othman.
\newblock Constructing event-driven partial barriers with resilience in
  wireless mobile sensor networks.
\newblock {\em Journal of Network and Computer Applications}, 82:77 -- 92,
  2017.

\bibitem{kingman}
J.F.C. Kingman.
\newblock {\em Poisson Process}, volume~3.
\newblock Oxford University Press, 1992.

\bibitem{spa_2013}
E.~Kranakis, D.~Krizanc, O.~Morales-Ponce, L.~Narayanan, J.~Opatrny, and
  S.~Shende.
\newblock Expected sum and maximum of displacement of random sensors for
  coverage of a domain.
\newblock In {\em Proceedings of the 25th ACM symposium on Parallelism in
  algorithms and architectures}, pages 73--82. ACM, 2013.

\bibitem{kumar2007}
S.~Kumar, T.~H. Lai, and A.~Arora.
\newblock Barrier coverage with wireless sensors.
\newblock {\em Wireless Networks}, 13(6):817--834, Dec 2007.

\bibitem{li2016}
F.~Li, J~Luo, S.~Xin, and Y.~He.
\newblock Autonomous deployment of wireless sensor networks for optimal
  coverage with directional sensing model.
\newblock {\em Computer Networks}, 108:120 -- 132, 2016.

\bibitem{8301580}
M.S. Miah and J.~Knoll.
\newblock Area coverage optimization using heterogeneous robots: Algorithm and
  implementation.
\newblock {\em IEEE Transactions on Instrumentation and Measurement},
  67(6):1380--1388, 2018.

\bibitem{mitrinovic}
D.~S. Mitrinovic.
\newblock {\em Analytic Inequalities}.
\newblock Springer, 1970.

\bibitem{mohamed2017}
S.M. Mohamed, H.S. Hamza, and I.A. Saroit.
\newblock Coverage in mobile wireless sensor networks (m-wsn): A survey.
\newblock {\em Computer Communications}, 110:133 -- 150, 2017.

\bibitem{oscar}
O.~Morales-Ponce.
\newblock Optimal patrolling of high priority segments while visiting the unit
  interval with a set of mobile robots.
\newblock In {\em Proceedings of the 21st International Conference on
  Distributed Computing and Networking}, ICDCN 2020, New York, NY, USA, 2020.
  Association for Computing Machinery.

\bibitem{GM-VPC}
Vishnu~G. Nair and K.~R. Guruprasad.
\newblock Gm-vpc: An algorithm for multi-robot coverage of known spaces using
  generalized voronoi partition.
\newblock {\em Robotica}, 38(5):845--860, 2020.

\bibitem{NIST}
NIST Digital~Library of~Mathematical~Functions.
\newblock http://dlmf.nist.gov/8.17.

\bibitem{complete_path}
Chu-Liang~L. Ping-Min~H. and Meng-Yao Y.
\newblock On the complete coverage path planning for mobile robots.
\newblock {\em Journal of Intelligent \& Robotic Systems}, 74:945--963, 2014.

\bibitem{ross2002}
S.~M. Ross.
\newblock {\em Probability Models for Computer Science}.
\newblock Academic press, 2002.

\bibitem{TIAN}
J.~Tian, X.~Liang, and G.~Wang.
\newblock Deployment and reallocation in mobile survivability-heterogeneous
  wireless sensor networks for barrier coverage.
\newblock {\em Ad Hoc Networks}, 36:321--331, 2016.

\bibitem{younis2008strategies}
M.~Younis and K.~Akkaya.
\newblock Strategies and techniques for node placement in wireless sensor
  networks: A survey.
\newblock {\em Ad Hoc Networks}, 6(4):621--655, 2008.

\bibitem{ZHOU2019}
Ch. Zhou, A.~Mazumder, A.~Das, K.~Basu, N.~Matin-Moghaddam, S.~Mehrani, and
  A.~Sen.
\newblock Relay node placement under budget constraint.
\newblock {\em Pervasive and Mobile Computing}, 53:1 -- 12, 2019.

\end{thebibliography}
\end{document}